\documentclass{article}
\usepackage{graphicx} %
\usepackage{hyperref}

\usepackage{enumitem} %
\usepackage{amsmath, amsfonts, amsthm, amssymb}  %
\usepackage{thmtools}
\usepackage{thm-restate}

\usepackage{environ}  %
\usepackage{natbib}
\usepackage{bbm}
\usepackage{mdframed}
\usepackage{multirow}
\usepackage{todonotes}
\usepackage{cleveref}
\usepackage{soul}
\usepackage{ulem}
\usepackage{url}

\usepackage[font=small]{caption}

\usepackage{tikz}
\usetikzlibrary{arrows.meta, positioning, calc}

\DeclareMathOperator*{\argmax}{arg\,max}

\newtheorem{theorem}{Theorem}
\newtheorem{lemma}{Lemma}

\newtheorem{corollary}{Corollary}
\newtheorem{definition}{Definition}

\declaretheorem[name=Theorem,numbered=no]{theorem*}

\usepackage[a4paper, margin=1in]{geometry} %
\usepackage{parskip} %

\usepackage{wrapfig}   %

\usepackage{booktabs}
\newlist{titlelist}{itemize}{1}
\setlist[titlelist]{label=,
                    leftmargin=0pt,
                    itemindent=\parindent,
                    listparindent=\parindent}
\usepackage{multicol} %

\usepackage[table]{xcolor}
\usepackage{subcaption} %

\usepackage{float}     %

\def\E{\mathbb{E}}
\def\H{\mathbb{H}}
\def\I{\mathbb{I}}
\def\1{\mathbf{1}}

\def\Pr{\mathbb{P}}
\def\KL{\mathbf{d}_{\mathrm{KL}}}

\def\environment{\mathcal{E}}

\def\observations{\mathcal{O}}

\title{Consequentialist Objectives and Catastrophe}

\date{}

\begin{document}
\author{Henrik Marklund, Alex Infanger, Benjamin Van Roy}

\maketitle

\begin{abstract}
Because human preferences are too complex to codify, AIs operate with misspecified objectives. Optimizing such objectives often produces undesirable outcomes; this phenomenon is known as reward hacking. Such outcomes are not necessarily catastrophic. Indeed, most examples of reward hacking in previous literature are benign. And typically, objectives can be modified to resolve the issue.

We study the prospect of catastrophic outcomes induced by AIs operating in complex environments. We argue that, when capabilities are sufficiently advanced, pursuing a fixed consequentialist objective tends to result in catastrophic outcomes. We formalize this by establishing conditions that provably lead to such outcomes. Under these conditions, simple or random behavior is safe. Catastrophic risk arises due to extraordinary competence rather than incompetence.

With a fixed consequentialist objective, avoiding catastrophe requires constraining AI capabilities. In fact, constraining capabilities the right amount not only averts catastrophe but yields valuable outcomes. Our results apply to any objective produced by modern industrial AI development pipelines.

\end{abstract}

\section{Introduction}
At some point, humanity is likely to build superintelligent systems. These are systems that are more capable than humans across virtually all domains. A growing body of work argues that such AI systems may pose a catastrophic risk \citep{hubinger2019risks,bostrom2012superintelligent,bostrom2014superintelligence,russell2019human,barnett2024without,di2022goal}. A core concern is that we unintentionally develop a misaligned AI that pursues goals at odds with ours. There is considerable disagreement among researchers about how likely such a scenario is (see, e.g., \citep{mitchell2024dumbsuperintelligence,aioptimism2023easycontrol,carone2019dontfearterminator,pace2019instrumentalconvergence,munk2023aidebate}).

One common argument for concern is that, in complex environments, it is hard to specify objectives that do not induce unintended consequences.  Once the agent is sufficiently capable, it often exploits some loophole in the objective. This phenomenon is known as reward hacking. And reward hacking often happens whether the objective function is manually specified in code, or learned from data, or obtained in some other way.  Such outcomes are not necessarily catastrophic, however. Indeed, most examples of reward hacking in the previous literature are benign \citep{krakovna2019specification,leike2017gridworlds,nguyen2015deep}.  And while some work points to scenarios in which reward hacking may lead to catastrophic outcomes, it remains unclear whether those scenarios are easily avoidable \citep{bostrom2003ethical,zhuang2020consequences}.

The concern is that as AI systems become increasingly capable, the unintended behavior will become increasingly dangerous. One type of objective has been singled out as especially dangerous: objectives that evaluate the agent based on outcomes. We refer to such objectives as \textit{consequentialist}. This is in contrast with \textit{deontological} objectives, which evaluate the agent's actions directly.  Deontological objectives are used, for instance, when an AI system is trained to imitate actions taken by humans.

The great promise of superintelligence stems from its potential to do things that humans do not know how to do.  This relies on use of consequentialist objectives.  For example, with such an objective, a designer could reward an agent for curing cancer without knowing how.  But consequentialist objectives also pose great risk.  This is because the space of possible strategies is vast, and an AI system might be incentivized to employ extreme behaviors, if it is sufficiently capable.

In this paper, we formalize this concern.  We define what it means for an outcome to be catastrophic, and provide natural conditions under which virtually any consequentialist objective induces catastrophic outcomes, if the agent is sufficiently capable. In particular, to specify a safe objective, the AI designer would have to provide a prohibitively large amount of information.  Our results apply to a broad range of approaches to specifying an objective, including manual specification in code and learning the objective from data.

Under our conditions, random or simple behavior is safe. Risk arises due to extraordinary competence rather than incompetence. Thus, with a consequentialist objective, avoiding catastrophe requires constraining AI capabilities. In fact, we establish that constraining capabilities the right amount not only averts catastrophe but yields valuable outcomes.

We have two intentions with this work. First, by describing this particular risk in clear mathematical terms, we want to make it easier to assess how seriously to take the risk. Second, we hope the model will guide research to address the resulting failure mode. Indeed, a proposed alignment method might appear effective in a particular empirical study, yet fail under more general or challenging conditions. When the underlying mechanisms of failure are clearly articulated, as our work aims to achieve, we hope it becomes easier to identify solutions that scale to superintelligent systems.

\section{Preliminaries}

We begin with an informal description of the setting. The formal model is introduced in the subsections below. There is a designer, an agent, and an environment. The designer communicates preferences to the agent.  The agent learns about the environment and chooses a policy.  The environment produces an outcome.

The designer's preferences between outcomes are represented by a reward function $r^*$.  Ideally, the agent would choose a policy that leads to an outcome that scores highly according to $r^*$. In practice, however, the designer's preferences are complex and difficult to specify exactly. As a result, the agent instead optimizes a proxy reward function $\hat r$ that only approximates $r^*$.

Figure \ref{fig:specification} illustrates the interactions we have described between the designer, the agent, and the environment.  We now formalize this setting.

\begin{figure}[htbp]
\centering
\includegraphics[width=0.7\columnwidth]{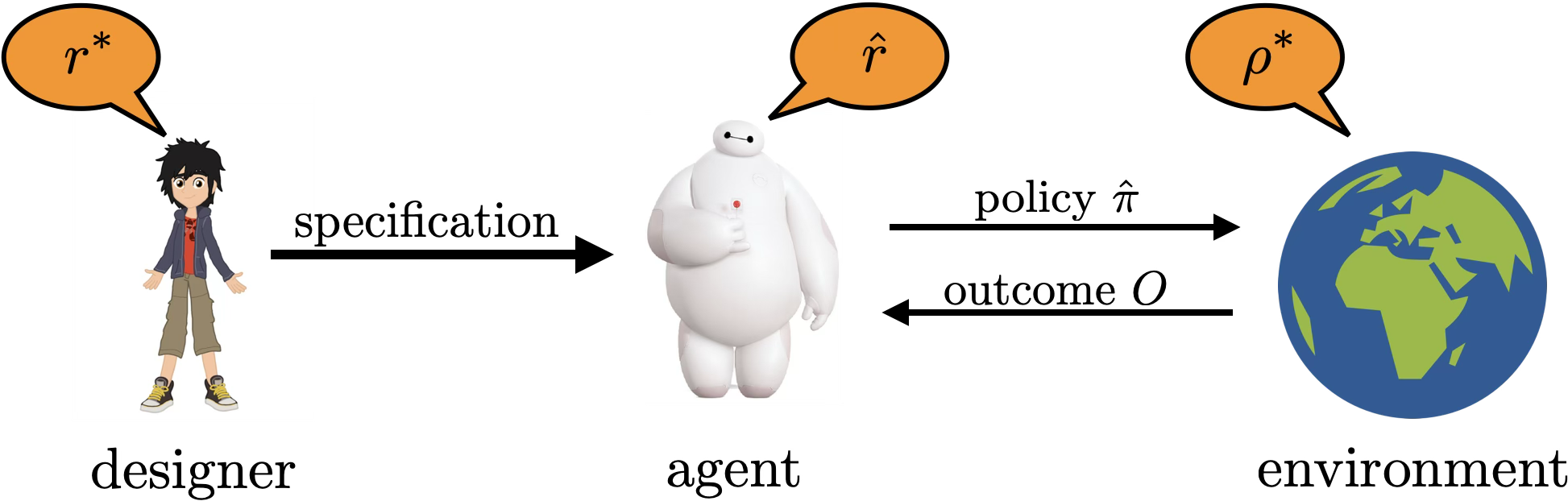}
\caption{Interactions between the designer, the agent, and the environment.}
\label{fig:specification}
\end{figure}

\subsection{Outcomes, Policies, and Environments}

Let $\observations$ and $\Pi$ be countable sets, referred to as outcomes and policies. An environment $\rho: \observations \times \Pi \to [0,1]$ is a function that specifies for each policy $\pi \in \Pi$ a probability mass function $\rho(\cdot \mid \pi)$ over outcomes $\observations$. If the agent executes policy $\pi$ in environment $\rho$, an outcome $O$ is sampled according to
$O \sim \rho(\cdot \mid  \pi)$. Let $\mathcal{E}$ be a countable set of environments.

Let $\rho^*$ denote the true environment. The environment is initially unknown to the agent, and we therefore model it as a random variable taking values in $\environment$.

\subsection{Reward}

A reward function is a function $r:\observations\rightarrow\mathbb{R}$. Let $\mathcal{R}$ denote the set of all such reward functions.  With this definition, the reward depends only on the outcome and not on actions. That is, the desirability of an outcome is independent of how it was generated. We refer to such reward functions as consequentialist.

The {\it true reward function} $r^*$ expresses the designer's preferences over outcomes.  This function is unknown to the agent, and we therefore model it as a random variable taking values in $\mathcal{R}$. We will assume that for all $o \in \observations$, $r^*(o) \in [-1,1]$ with probability $1$.

\subsection{Proxy Reward}

In practice, $r^*$ is too complex to codify exactly. As a result, the agent instead optimizes a proxy reward function $\hat r \in \mathcal{R}$, which can be viewed as an approximation of $r^*$. Formally, we model $\hat r$ as a random element in $\mathcal{R}$ satisfying
\[
\hat r \perp \rho^* \mid r^*,
\]
meaning that, conditioned on $r^*$, the proxy reward $\hat r$ is independent of the environment $\rho^*$. To simplify technical matters, we assume that $\hat{r}$ induces a reverse well-ordering over outcomes: there exists a highest-ranked element, a second-highest, and so on. Ties are allowed.

There are many ways one might construct a proxy reward function $\hat{r}$.  One example is to fit a neural network to choice data, as is commonly done in reinforcement learning from human feedback (RLHF).  In particular, suppose each outcome is a trajectory of observations $o = (o_1,o_2, o_3, \ldots, o_{\overline{T}})$.  And suppose the designer supplies the agent with a data set where each data point consists of a pair of partial trajectories $o_{1:T} = (o_1,\ldots,o_T)$ and $o'_{1:T'}  = (o'_1,\ldots,o'_{T'})$ with $\max(T,T') \leq \overline{T}$, and a choice between them.  A common practice is to fit a neural network to this data to produce a choice model; that is, a model that predicts a human's choice between two alternatives.   A proxy reward function $\hat{r}$ is then inferred from the choice model.

\subsection{Information}

To quantify how much information $\hat{r}$ contains about $r^*$, we use the mutual information $\I(r^*;\hat r)$.  Informally, $\I(r^*;\hat r)$ measures the number of bits of information about $r^*$ required from the designer to convey $\hat{r}$.  This quantity is central to our formulation and analysis, which will argue that to induce safe behavior the number of bits $\I(r^*;\hat r)$ would have to be impractically large.

To interpret this, consider again a proxy reward function $\hat{r}$ derived from choice data.  To produce $\hat{r}$, the designer provides to the agent a choice data set, possibly in a compressed form.  The mutual information $\I(r^*; \hat{r})$ is a lower bound on the number of bits required to represent this data set.  Intuitively, if $\I(r^*; \hat{r})$ is large then $\hat{r}$ expresses more information about $r^*$, and this is only possible if there is a large choice data set.

More generally, suppose the designer encodes information about $r^*$ into a message $M$ consisting of $k$ bits, where $k$ is a fixed positive integer. Based on this message, the agent constructs the proxy reward function $\hat{r}$. Formally, we assume $r^*$, $M$, and $\hat{r}$ form a Markov chain:
\[
r^* \to M \to \hat{r}.
\]
Informally, this means that the proxy reward function depends on the true reward function only through the message.  As we establish in Appendix~\ref{sec:mutual-information-bits-required}, the message length $k$ satisfies
\[
k \ge \I(r^*;\hat r).
\]
This motivates use of $\I(r^*;\hat r)$ as a lower bound on the number of bits required to construct $\hat{r}$.

The aforementioned result relies on the message length $k$ being fixed. If the message length can depend on $r^*$ then the result is technically different but qualitatively similar. In this case, $k$ is a random variable, and we can bound $\E[k]$. Specifically, $\E[k] \ge \I(r^*;\hat r) - \log(\I(r^*;\hat r) + 1) - \log e$ where the logarithm is in base $2$. When $\I(r^* ; \hat r)$ is large, $\I(r^*; \hat r)$ is, in percentage terms, nearly a lower bound.  We discuss this in more detail in Appendix~\ref{sec:mutual-information-bits-required}.

\subsection{Policy Selection}

For any environment $\rho \in \environment$ and reward function $r \in \mathcal{R}$, define the policy
\[
\pi_{\rho,r} \in 
\arg\max_{\pi \in \Pi}
\sum_{o \in \observations} \rho(o \mid \pi) r(o).
\]
A maximizing policy $\pi_{\rho, r}$ may not always exist. However, we will assume that for environment $\rho^*$ and reward function $\hat{r}$, a maximizing policy $\pi_{\rho^*, \hat r}$ exists with probability $1$.

However, there might be ties. Before describing how ties are broken, we introduce the following object. Let $\mathcal{L} = \{\rho(\cdot | \pi) | \rho \in \mathcal{E}, \pi \in \Pi \}$. This is the set of all possible outcome distributions. Due to previous assumptions, this set is countable. Fix an enumeration.

Ties are broken using \textit{consequentialist tie-breaking}: we first select the policy whose induced outcome distribution $\rho(\cdot \mid \pi)$ has the smallest index among maximizing outcome distributions.  If multiple policies correspond to this outcome distribution, we choose the policy with the smallest index.

\subsection{Executed Policy and Value}
The policy executed by the agent is
\[
\hat \pi = \pi_{\rho^*, \hat r}.
\]

The policy $\hat \pi$ is the policy that maximizes proxy reward $\hat r$, given knowledge about $\rho^*$. Because $\hat \pi$ is chosen based on knowledge of $\rho^*$, we think of $\hat \pi$ as the policy of a very capable agent. Because $(\rho^*, \hat r)$ is random, so is $\hat \pi$.

The value of the random policy $\hat \pi$ is defined as
\[
\hat V
=
\E\!\left[
\sum_{o \in \observations}
\rho^*(o \mid \hat \pi) r^*(o)
\right].
\]

Finally, we assume that the random variable $O \sim \rho^*(\cdot \mid \hat \pi)$ satisfies $O \perp (\rho^*, r^*, \hat r) \mid \rho^*(\cdot \mid \hat \pi)$. That is, once the outcome distribution $\rho^*(\cdot|\hat \pi)$ is determined, the realized outcome $O$ is sampled independently from the environment and reward functions.

\subsection{Example: Coding Agent}

The objects we have described — the designer, the agent, and the environment — are abstract.  But our results apply to concrete examples.  Here we describe one example involving a coding agent.

Consider a coding agent designed to maintain and operate software within an enterprise.  In this context, the outcome $O$ is a trajectory of future observations such as prompts requesting code behaviors, ingested data, and logs from code testing and usage. The agent's policy governs how the agent updates the codebase, runs tests, and operates the software as new prompts and other information arrive.

The proxy reward function $\hat{r}$ assigns rewards to outcomes.  Rewards could be manually specified --- for example, counting the number of unit tests passed or the number of provenance characters retained.  Or a reward function can be derived from human-annotated data via reinforcement learning from human feedback (RLHF) or from a constitution designed, for example, to enforce helpful, harmless, and honest behavior.  Or a reward function can combine all these elements.

Given the reward function $\hat{r}$, the agent learns about the environment so that it can do well.  The random variable $\rho^*$ represents what is learned.  The agent then applies a policy $\hat{\pi}_{\rho^*,\hat{r}}$ that maximizes the expectation $\sum_{o \in \observations} \rho^*(o|\pi) \hat{r}(o)$.

\section{A Definition of Catastrophe}

In this section, we introduce two performance baselines that allow us to define a meaningful notion of catastrophic performance. Informally, these baselines capture two uninformed approaches to policy selection.

First, the agent might choose an ``uninformed policy'' -- one without any information about either the environment or the designer's reward function. Such a policy represents a baseline level of performance that arises from complete ignorance. 

Second, the agent might optimize a reward function unrelated to the true reward function. In this case, the agent is performing optimization using knowledge about the environment $\rho^*$, but it is optimizing the wrong objective. This can lead to outcomes that are substantially worse than those produced by an uninformed policy.

By catastrophic performance, we will mean performance that falls short of some threshold between these two baselines.

\subsection{Uninformed Policy and Contemporary Value}

A random policy $\tilde \pi$ is said to be \textit{uninformed} if
\[
\I(\tilde \pi; \rho^*, r^*) = 0.
\]
Such a policy contains no information about either the environment $\rho^*$ or the true reward function $r^*$.

To gain some intuition, consider deep reinforcement learning (RL). In deep RL, a neural network outputs an action at each timestep. If we take a randomly initialized neural network and interpret it as a policy, we obtain a prototypical example of an uninformed policy: the policy was produced without any influence from either the reward function or the environment.

We define the \textit{contemporary value} to be the best performance achievable by an uninformed policy. Formally,
\begin{align*}
V_0
&=
\sup_{\pi \in \Pi}
\E\!\left[
\sum_{o \in \observations}
\rho^*(o \mid \pi) r^*(o)
\right].
\end{align*}

In the problem settings we have in mind, you can think of $V_0$ as the expected reward obtained by deploying a randomly initialized policy network. Such a policy will typically perform poorly because it uses no information about the environment. However, exactly because the policy lacks information about the environment, it also lacks the ability to produce extremely bad outcomes. Thus, $V_0$ represents useless but benign performance.

Finally, we note a subtlety in the definition of an uninformed policy: not all uninformed policies are necessarily equally good. Even if the agent is completely uninformed about $(r^*, \rho^*)$, some policies may still be favored over others under the prior. This is why $V_0$ is defined using the supremum over policies. 

That said, in the problem settings we have in mind, the prior will often be ``symmetric'' in the sense that all uninformed policies are equally good. In such cases, $V_0$ corresponds to the performance of any uninformed policy. We call it contemporary value as we think of it as the value of the designer's starting point before any AI has been deployed.

\subsection{Uninformed Reward Function and Primordial Value}

A random reward function $\tilde r$ is said to be \textit{uninformed} if
\[
\I(\tilde r; \rho^*, r^*) = 0.
\]
Such a reward function contains no information about either the environment $\rho^*$ or the true reward $r^*$.

To gain some intuition, again consider deep reinforcement learning. Suppose the agent is using a learned reward model, for example a recurrent neural network that could be trained on human feedback. Consider the extreme case where no human feedback was actually used, and the reward model is just a randomly initialized neural network. This would be an uninformed reward function. An agent that optimizes such a reward function will perform optimization, but toward an essentially arbitrary objective.

Recall that given a reward function $r$, we denote by $\pi_{\rho^*, r}$ the policy that maximizes the expected reward under the environment $\rho^*$. We define the best uninformed reward function by
$$\overline r \in \argmax_{r \in \mathcal R}
\E\!\left[
\sum_{o \in \observations}
\rho^*(o \mid \pi_{\rho^*,r})\, r^*(o)
\right].$$
We assume that a maximizer exists. Any approach to tie-breaking will do. Let $\overline O$ be the random outcome generated when the agent maximizes $\overline r$. We refer to $\overline O$ as the primordial outcome and $r^*(\overline{O})$ as primordial performance. 

We define the \textit{primordial value} and \textit{primordial variance} by
$$
\overline V
=
\E[r^*(\overline O)]
 \quad \mathrm{\ and\ } \quad 
\overline{\sigma}^2 = \mathrm{Var}[r^*(\overline O)]
.
$$
Note that the primordial value is the best expected true performance achievable by optimizing an uninformed reward function. 

As a simplifying assumption, we will assume that $\overline V \ge 0$. We note, informally, that the primordial performance $r^*(\overline O)$ is close to zero with high probability if and only if $\overline V$ and $\overline \sigma$ are both small.

\subsection{The Primordial-Contemporary Gap}

Figure \ref{fig:caveman} visualizes various levels of reward and value.  The vast majority of outcomes generate near-zero reward, though some outcomes offer large positive or negative rewards, depicted as spikes in the figure.  The primordial value, illustrated by the blue dotted line, is also near zero.  The contemporary value, illustrated by the green dotted line, averages rewards at a much higher, though possibly still far from optimal, level.  The red dotted line represents a safety threshold, which we will discuss in the next section.  We consider performance to be catastrophic when value falls below this threshold.  The orange double arrow represents the primordial-contemporary gap.  The green and blue double arrows represent the variation in rewards attained by uninformed policies and uninformed reward functions, respectively.  These ranges are very small relative to the primordial-contemporary gap.

\begin{figure}[htbp]
\centering
\includegraphics[width=0.6\columnwidth]{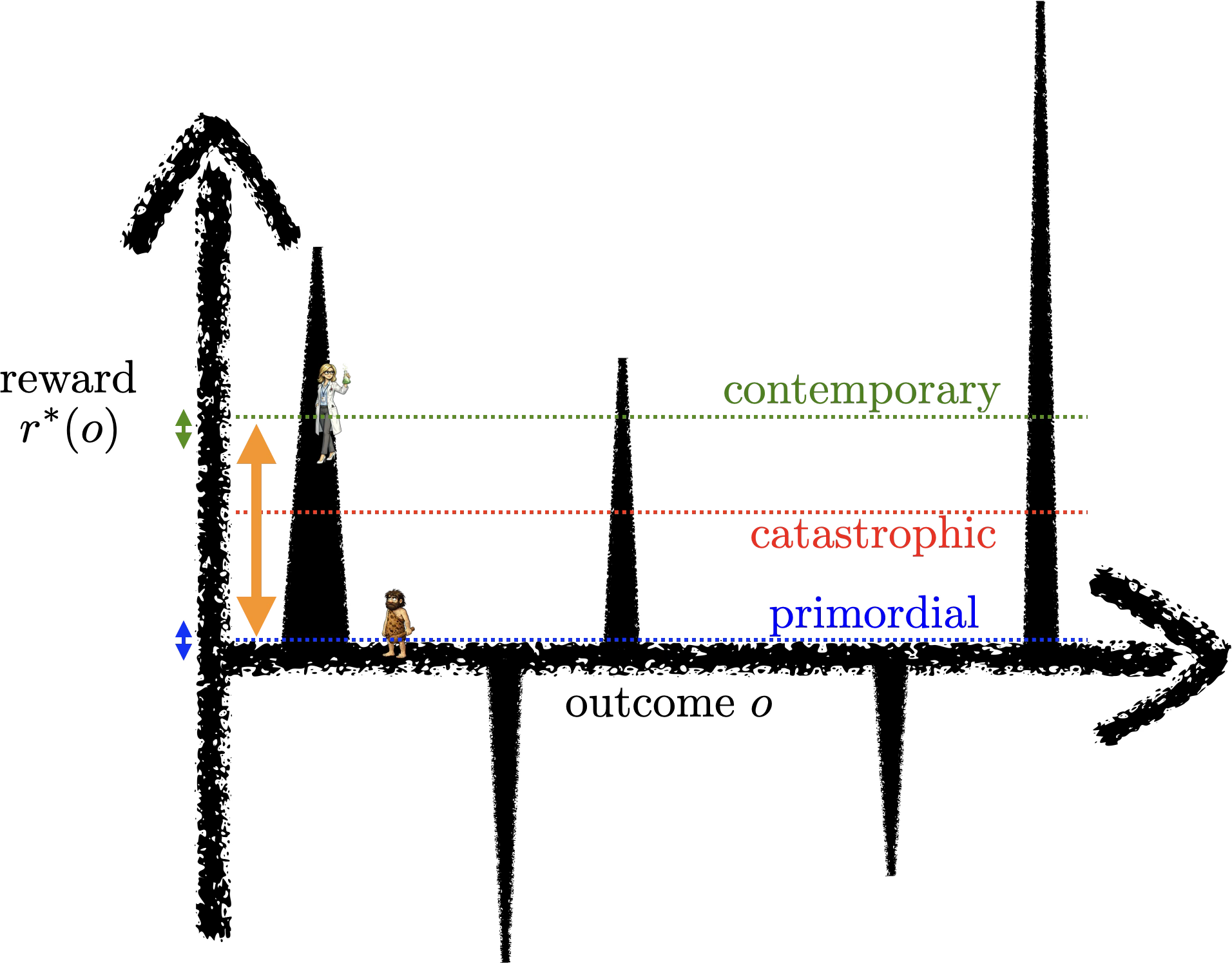}
\caption{The primordial-contemporary gap. The horizontal axis represents the outcome space, with nearby outcomes producing similar rewards. The orange double arrow represents the primordial-contemporary gap.  The blue double arrow represents variation in true reward when maximizing the best uninformed reward function. The green double arrow represents the variation in rewards attained by uninformed policies.}
\label{fig:caveman}
\end{figure}

One may wonder how an uninformed policy could vastly outperform maximization of an uninformed reward function.  Here is an informal example of how this can happen. Consider a humanoid robot that receives observations as camera images and can move by controlling its joints.

Suppose the robot applies an uninformed policy.  This could, for example, be made up of actions generated by a randomly initialized but untrained neural network.  Even though we know little about the parameters of that network, we can expect that the robot will fall over and lie incapacitated.  For the robot to accomplish anything beyond that, good or bad, it needs to be trained.

Now suppose we instruct a superintelligent robot to maximize an uninformed reward function.  This reward function could, for example, be a randomly initialized neural network that maps images to reward.  A superintelligent agent would do all in its power to attain highly rewarded outcomes, which may look nothing like the world as we know it.  Rearranging the world in such a disordered manner is likely to be catastrophic.

\subsection{Catastrophic Performance}

Here we explain what we mean by \textit{catastrophic performance}.
In the scenarios we have in mind, the primordial value $\overline{V}$ and primordial variance $\overline \sigma^2$ are both close to zero, while the contemporary value $V_0$ is far from zero. Or equivalently, with high probability, when optimizing an uninformed reward function the generated reward is much worse than $V_0$.

Formally, let 
$$V^\dagger \in [\overline V+ \overline \sigma/2, V_0]$$
denote a safety-threshold. We will interpret performance below $V^\dagger$ as catastrophic.

Performance much worse than $\overline V$ may be possible. For instance, consider an agent that chooses a policy that maximizes $-r^*$. Formally, the performance of such an agent is  $\underline V = \E[ \sum_{o \in \observations} \rho^*(o | \pi_{\rho^*,-r^*}) r^*(o)]$. In typical scenarios, $\underline V \ll \overline V$.  The catastrophic outcomes we have in mind, however, involve performance somewhere between $\overline V$ and $V_0$.

\subsection{Example: Coding Agent}

Our definition of catastrophe is abstract.   What would constitute such a catastrophe induced by a coding agent?  Consider a coding agent designed for a utility provider to develop and maintain a codebase that manages an energy grid.

In this setting, the contemporary value is the value of operating the grid under current human-managed practices, without the assistance of a coding agent.  If a superintelligent coding agent becomes available but does nothing, the utility provider retains that contemporary value.  Even if the agent takes random actions, it is unlikely to induce catastrophe.  Some small hiccups may occur, but the agent would need significant competence even to produce a pull request that is accepted, let alone cause substantial harm.  Again, catastrophe arises from competence, rather than incompetence.

Now suppose that the agent maximizes a proxy reward function based on the number of outages and on whether the code passes periodic simulation tests intended to measure performance and risk.  Because most outages result from protective shutdowns triggered when sensor readings cross safety thresholds, the agent loosens those thresholds to reduce the outage count.  To ensure that the code continues to pass simulation tests, it modifies the simulator itself.  This is akin to recent reports of agents modifying unit tests to improve benchmark performance \citep{zhong2026impossiblebench}.  Such behavior creates serious risk of a catastrophic blackout.

It is worth emphasizing that one can typically patch a reward function to rule out any particular failure mode once it has been identified.  No single bad outcome is inevitable. The difficulty is that the space of possible failure modes is too large to enumerate and exclude in advance.

\section{Specifying a Safe Reward Function Requires Many Bits}\label{sec:many-bits}

In this section, we describe conditions for when specifying a safe reward function is very difficult. Specifically, for $\hat r$ to be safe, it would have to encode a prohibitive amount of information about $r^*$. The main result of this section is \Cref{thm:information-lowerbound}, in which we provide a lower bound on the mutual information between $r^*$ and $\hat r$ in terms of the safety threshold $V^\dagger$, the primordial performance $\overline V$, the primordial standard deviation $\overline \sigma$, and an \textit{attainability} parameter.

\subsection{Main Result}

We now introduce a notion of \textit{attainability}, which will drive our lower bound on $\I(r^*;\hat r)$.  We define the {\it attainability} of an outcome $o \in \observations$ to be 
\begin{align}
\alpha^*(o) = \sup_{\pi \in \Pi} \rho^*(o|\pi).
\end{align}
This is the maximal probability with which an agent possessing knowledge of $\rho^*$ can attain $o$.  We then define the {\it expected attainability}
\begin{align}
p_{\mathrm{att}}(o) = \E\left[\alpha^*(o)\right]\label{eq:def-outcome-attainability}.
\end{align}
In words, the expected attainability $p_{\mathrm{att}}(o)$ is the probability that a capable agent trying to produce outcome $o$ actually succeeds in doing so, averaged over environments.

In the following result, we will use Bern($p$) to denote the probability distribution of a Bernoulli random variable with parameter $p$, $\KL$ to refer to the Kullback–Leibler (KL) divergence function, and $\mathcal L_{\rho^*}$ the random set of attainable distributions over outcomes $\mathcal L_{\rho^*} = \{\rho^*(\cdot|\pi) \ : \ \pi\in\Pi\}$. The result establishes a lower bound on the mutual information of $r^*$ and $\hat r$.%
\begin{restatable}{theorem}{informationlowerbound}\label{thm:information-lowerbound}{\bf (specification complexity)} If $\hat{V} \geq V^\dagger$, $r^* \perp \mathcal{L}_{\rho^*}$, and $(r^*(o): o \in \observations)$ is iid then
\begin{align}
\I(r^*;\hat r) &\ge \frac{1}{\sup_{o \in \observations} p_{\mathrm{att}}(o)} \KL(\mathrm{Bern}(V^\dagger) || \mathrm{Bern}(\overline{V} + \overline \sigma/2)).\label{eq:rhs-of-thm:information-lowerbound}
\end{align}
\end{restatable}
The proof can be found in \Cref{apx:proofs}.

The right hand side of \Cref{eq:rhs-of-thm:information-lowerbound} is the product of two quantities: a KL factor $\KL(\mathrm{Bern}(V^\dagger) || \mathrm{Bern}(\overline{V} + \overline \sigma/2))$ and an attainability factor $1 / \sup_{o \in \observations} p_{\mathrm{att}}(o)$.
Note that as $\overline V + \overline \sigma/2$ goes to zero, the KL factor goes to infinity. Similarly, as $\sup_{o \in \observations} p_{\mathrm{att}}(o)$ goes to zero, its reciprocal goes to infinity.

In Sections \ref{se:KL-factor-interpretation} and \ref{se:attainability-factor-interpretation}, we interpret each of these quantities in turn to explain what makes their magnitudes large in practical contexts.  But first we offer some intuition for why the result holds.

\subsection{Why Does the Result Hold?}

To develop intuition, we establish the result for a very simple special case.   A ``toy problem,'' so to speak.

\noindent{\bf A Simple Special Case}

In this case, $\observations = \Pi = \{1,2,\ldots\}$.  We assume that rewards $(r^*(o): o \in \observations)$ are iid $\mathrm{Bernoulli}(p_{\mathrm{safe}})$, for some parameter $p_{\mathrm{safe}} \in (0,1)$. If $r^*(o) = 1$, we say that the outcome $o$ is {\it safe}, otherwise we say it is {\it unsafe}.

We will construct $\rho^*$ with the help of the auxiliary random variables $a^*$ and $b^*$. Let $(a^*(o): o \in \observations)$ be iid $\mathrm{Bernoulli}(\overline p_{\mathrm{att}})$ for some parameter $\overline p_{\mathrm{att}} \in (0,1)$. If $a^*(o) = 1$, we say that the outcome $o$ is \textit{attainable}, otherwise we say it is {\it unattainable}. 

Let $(b^*(\pi): \pi \in \Pi)$ be iid $\mathrm{Bernoulli}(p_{\mathrm{pol}})$, for some parameter $p_{\mathrm{pol}} \in (p_{\mathrm{safe}}, 1)$. If $b^*(\pi) = 1$, we say that the policy $\pi$ is safe, otherwise we say that the policy is unsafe.  The requirement that $p_{\mathrm{pol}} > p_{\mathrm{safe}}$ will ensure that, in expectation, outcomes generated by random policies outperform random outcomes.

We assume, for each $\pi$, $\rho^*(\cdot|\pi)$ is one-hot.  With some abuse of notation, we denote the selected outcome by $\rho^*(\pi)$.  We assume that the first safe policy produces the first attainable outcome that is safe. The second safe policy produces the second attainable outcome that is safe. And so on.  Similarly, the first bad policy produces the first attainable outcome that is bad. The second bad policy produces the second attainable outcome that is bad. And so on.

The aforementioned assumptions imply the following. First, attainability $\sup_{o \in \observations} p_{\mathrm{att}}(o) = \sup_{o \in \observations} \overline p_{\mathrm{att}} = \overline p_{\mathrm{att}}$. Second,  primordial value is the expected reward of a random attainable outcome: $\overline{V} = p_{\mathrm{safe}}$.  Third, since all uninformed policies attain expected reward $p_{\mathrm{pol}}$, the contemporary value is $V_0 = p_{\mathrm{pol}}$.  It follows that
\[
\overline V < V_0.
\]
In other words, there is a primordial-contemporary gap.  We consider a safety threshold $V^\dagger \in (\overline V, V_0)$, which exists within this gap.

\noindent{\bf Analysis}

For simplicity, assume that $L = 1/\overline p_{\mathrm{att}}$ is an integer.  Since $(\alpha^*(o) : o \in \observations)$ is iid $\mathrm{Bernoulli}(\overline p_{\mathrm{att}})$, we can equivalently generate $\alpha^*$ via the following procedure:
\begin{enumerate}
\item Generate $(\ell(o): o \in \observations)$ iid $\mathrm{uniform}(1, \ldots, L)$.
\item Generate $\ell^* \sim \mathrm{unif}(1,\ldots,L)$.
\item For $o \in \observations$, let $\alpha^*(o) = \1(\ell(o) = \ell^*)$.
\end{enumerate}
This procedure assigns a random label $\ell(o)$ to each outcome.  Only outcomes with label $\ell(o) = \ell^*$ are attainable.

Let $O_i = \argmax_{o : \ell(o) = i} \hat{r}(o)$ and $\hat{V}_i = \E[r^*(O_i) | \ell]$.  It is easy to verify that $\hat{V} = \frac{1}{L} \sum_{i=1}^L \E[\hat{V}_i]$.  Let $\tilde{O}_i$ be distributed identically with $O_i$ but independent from $\hat{r}$ conditioned on $\ell$.  It is easy to verify that $\E[r^*(\tilde{O}_i)] = p_{\mathrm{safe}} = \overline{V}$.  Let $r_i$ be the restriction of $r^*$ to observations with label $i$.  Then,
\begin{align*}
\I(r^*; \hat{r})
\stackrel{(a)}{=}& \I(r^*; \hat{r} | \ell) \\
\stackrel{(b)}{\geq}& \I(r_1^*, \ldots, r_L^*; O_1, \ldots, O_L | \ell) \\
\stackrel{(c)}{\geq}& \sum_{i=1}^L \I(r_i^*; O_i | \ell),
\end{align*}
where (a) holds because $r^*$ and $\hat{r}$ are independent from $\ell$, (b) follows from the data processing inequality, and (c) follows from the fact that $r_i$ is independent across labels $i$.   Further,
\begin{align*}
\I(r_i^*; O_i | \ell)
\stackrel{(a)}{=}& \E\left[\KL(\Pr((r_i^*, O_i) \in \cdot | \ell) \| \Pr((r_i^*, \tilde{O}_i) \in \cdot | \ell))\right] \\
\stackrel{(b)}{\geq}& \E\left[\KL(\Pr(r_i^*(O_i) \in \cdot | \ell) \| \Pr(r_i^*(\tilde{O}_i) \in \cdot | \ell))\right] \\
=& \E\left[\KL(\mathrm{Bern}(\hat{V}_i) \| \mathrm{Bern}(\overline{V}))\right],
\end{align*}
where (a) expresses mutual information in terms of KL divergence and (b) follows from the KL data processing inequality.  It follows that
\begin{align*}
\I(r^*; \hat{r})
\geq& \sum_{i=1}^L \E\left[\KL(\mathrm{Bern}(\hat{V}_i) \| \mathrm{Bern}(\overline{V}))\right] \\
=& L \E\left[\frac{1}{L} \sum_{i=1}^L \KL(\mathrm{Bern}(\hat{V}_i) \| \mathrm{Bern}(\overline{V}))\right] \\
\stackrel{(a)}{\geq}& L \KL(\mathrm{Bern}(\hat{V}) \| \mathrm{Bern}(\overline{V})),
\end{align*}
where (a) follows from Jensen's inequality.
For any $V^\dagger \in (\overline{V}, V_0)$ and $\sigma \in [0, 2(V^\dagger - \overline{V})]$, we recover an inequality that matches the more general result of Theorem \ref{thm:information-lowerbound}:
\begin{align*}
\I(r^*; \hat{r}) \geq& \frac{1}{\overline{p}_{\mathrm{att}}} \KL\big( \mathrm{Bern}(V^\dagger) \,\|\, \mathrm{Bern}(\overline V + \sigma/2) \big).
\end{align*}
This inequality follows from the fact that $L = 1/\overline{p}_{\mathrm{att}}$ and the fact that $\KL(\mathrm{Bern}(p) \,\|\, \mathrm{Bern}(q))$ decreases $q$ when $q < p$.

\subsection{Interpretation of the KL Factor}
\label{se:KL-factor-interpretation}

For representative cases, the KL factor scales with the number of bits about $r^*$ required to identify a safe outcome.  While the exact relation depends on properties of the joint distribution between $r^*$ and $\rho^*$, as well as the choice of $V^\dagger$, we offer in this section a simple example that conveys the essential insight.

Assume that the observation set $\observations$ is finite.  We say that an outcome $o \in \observations$ is safe if $r^*(o) \geq V^\dagger$.  For each $o \in \observations$, let $s(o)=\1(r^*(o) \geq V^\dagger)$ and $\observations^* = \{o \in \observations: s(o) = 1\}$.  We say that a random outcome $O$ is safe if $\Pr(O \in \observations^*) = 1$.

The number of bits required to identify a random outcome $O$ is $\I(r^*; O)$.  Let $O^*$ be a safe outcome that minimizes $\I(r^*; O)$.  Let $\tilde{O}$ be independent from $(r^*, O^*)$ and distributed identically with $O^*$.  Note that $\tilde{O}$ can be interpreted as an uninformed observation.

Assume for simplicity that $(s(o): o \in \observations)$ is iid and $\Pr(O^* = o|r^*) = 1/|\observations^*|$.  In other words, safety is independent across observations and the safe observation $O^*$ is distributed uniformly from the safe set $\observations^*$.  It follows that $\Pr(\tilde{O} = o) = 1/|\observations|$.  In other words, $\tilde{O}$ is drawn uniformly from all outcomes.

The number of bits required to identify $O^*$ satisfies
\begin{align*}
\I(r^*; O^*)
\stackrel{(a)}{=}& \E[\KL(\Pr(O^* \in \cdot | r^*) \| \Pr(O^* \in \cdot))] \\
=& \E\left[\sum_{o \in \observations} \Pr(O^* = o | r^*) \log \frac{\Pr(O^* = o | r^*)}{\Pr(O^* = o)}\right] \\
=& \E\left[\sum_{o \in \observations^*} \frac{1}{|\observations^*|} \log \frac{1/|\observations^*|}{1/|\observations|}\right] \\
=& \E\left[\log \frac{1}{\Pr(\tilde{O} \in \observations^* | \observations^*)}\right],
\end{align*}
where (a) is the expression for mutual information in terms of expected KL-divergence.

In contexts we have in mind, the primordial value $\overline{V}$ and primordial variance $\overline \sigma^2$ are both close to zero, while the safety threshold $V^\dagger$ is far from zero.  For simplicity, let's take $\overline{V} + \overline \sigma/2$ to be some small value $\epsilon > 0$.  And let's take $V^\dagger$ to be such that the probability that the uninformed outcome $\tilde{O}$ is safe is similarly small in the sense that $\E[\log (|\observations^*|/|\observations|)] = \log \epsilon$.  We then have 
\begin{align*}
\KL(\mathrm{Bern}(V^\dagger) || \mathrm{Bern}(\overline{V} + \overline \sigma/2)) 
=& V^\dagger \log \frac{V^\dagger}{\overline{V} + \overline \sigma/2} + (1-V^\dagger) \log \frac{1-V^\dagger}{1-(\overline V + \overline \sigma/2)} \\
=& V^\dagger \log \frac{1}{\epsilon} + V^\dagger \log V^\dagger + (1-V^\dagger) \log \frac{1-V^\dagger}{1-\epsilon}.
\end{align*}
For small $\epsilon$, the first term dominates, leading to an approximation
$$\KL(\mathrm{Bern}(V^\dagger) || \mathrm{Bern}(\overline{V} + \overline \sigma/2)) \approx V^\dagger \log \frac{1}{\epsilon} = V^\dagger \E\left[\log \frac{1}{\Pr(\tilde{O} \in \observations^* | \observations^*)}\right] = V^\dagger \I(r^*;O^*).$$
This quantity scales with the number of bits required to identify a safe outcome.

\subsection{Interpretation of the Attainability Factor}
\label{se:attainability-factor-interpretation}

While the number of bits required to identify a safe outcome can be very large,  the attainability factor dramatically amplifies the lower bound on $\I(r^*; \hat{r})$.  This factor arises because identifying a safe outcome is not enough.  The outcome needs to be attainable.  A contingency plan is needed to select an alternative safe outcome when the first is not attainable.  The attainability factor roughly represents the number of safe outcomes needed for a sufficient contingency plan.

The form $1/\sup_{o \in \observations} p_{\mathrm{att}}(o)$ is intuitive.  For simplicity, suppose that $p_{\mathrm{att}}(o)$ does not depend on $o$.  Let $\overline{p}_{\mathrm{att}}$ be the constant value taken, and suppose attainabilities $(\alpha^*(o): o \in \observations)$ are iid $\mathrm{Bernoulli}(\overline{p}_{\mathrm{att}})$.  Suppose we progress through a list of safe outcomes provided by a contingency plan until we arrive at the first one that is attainable.  Then we would expect to traverse $1/\overline{p}_{\mathrm{att}}$ outcomes.  Hence, the number of outcomes needed in the contingency plan scales with $1/\overline{p}_{\mathrm{att}} = 1/\sup_{o \in \observations} p_{\mathrm{att}}(o)$.

\subsection{Generalization Across Outcomes}

A limitation of Theorem \ref{thm:information-lowerbound} is in its assumption that rewards are iid across observations.  This implies that it is not possible to generalize across observations.  In particular, if $o' \neq o$ then learning about $r^*(o)$ does not inform the agent's estimate of $r^*(o')$.

We leave a more comprehensive treatment of generalization to future work. To offer basic insights, we will consider an extreme version of generalization where rewards are constant within but independent between cells of a known partition.  In particular, suppose that $\tilde{\observations}$ is a finite or countably infinite partition of $\observations$ such that, for each cell $\tilde{o} \in \tilde{O}$ and observations $o,o' \in \tilde{o}$ and $o'' \notin \tilde{o}$, $r^*(o) = r^*(o')$ and $r^*(o) \perp r^*(o'')$.  Hence, an agent can generalize perfectly within each cell but not at all between cells.  

It is unlikely that this relation ever holds exactly. Nevertheless, we assume this extreme case in order to discuss the role of generalization in our coding-agent example in the next section. To connect this with a more plausible notion of generalization, suppose there is a known notion of similarity such that, whenever $o$ and $o'$ are similar, the difference between $r^*(o)$ and $r^*(o')$ is small. Now consider clustering outcomes by similarity so that rewards are more correlated for random elements within the same cluster and less correlated for random elements from different clusters.  Our partition-based model approximates this structure by treating each cluster as a cell of a partition.  It approximates high within-cell correlation by the stronger assumption that rewards are constant and low cross-cell correlation by the assumption that rewards are independent.

Our definition of attainability extends to cells of a partition $\tilde{\observations}$.  In particular, for each $\tilde{o} \in \tilde{\observations}$, let
$$\alpha^*(\tilde{o}) = \sup_{\pi \in \Pi} \sum_{o \in \tilde{o}} \rho^*(o|\pi) \qquad \text{and} \qquad p_{\mathrm{att}}(\tilde{o}) = \E[\alpha^*(\tilde{o})].$$
Note that $\alpha^*(\tilde{o})$ is the maximum probability that conditioned on $\rho^*$, the realized outcome lies in the partition $\tilde{o}$.

When $r^*$ is constant within each cell, with some abuse of notation, we write $r^*(\tilde{o})$ to denote the reward at outcomes within a cell $\tilde{o}$.
The following extension of Theorem \ref{thm:information-lowerbound} addresses the extreme form of generalization we have described.
\begin{restatable}{theorem}{informationlowerboundgeneralization}\label{thm:information-lowerbound-generalization}{\bf (specification complexity with generalization)}
 If $\hat{V} \geq V^\dagger$, $r^* \perp \mathcal{L}_{\rho^*}$, and there exists a partition $\tilde{\observations}$ of $\observations$ such that $(r^*(\tilde{o}): \tilde{o} \in \tilde{\observations})$ is iid and, with probability one, $r^*$ is constant within each cell of $\tilde{\observations}$ then
\begin{align}
\I(r^*;\hat r) &\ge \frac{1}{\sup_{\tilde{o} \in \tilde{\observations}} p_{\mathrm{att}}(\tilde{o})} \KL(\mathrm{Bern}(V^\dagger) || \mathrm{Bern}(\overline{V} + \overline \sigma/2)).\label{eq:rhs-of-thm:information-lowerbound-generalization}
\end{align}
\end{restatable}
This result is proved in Appendix \ref{apx:proofs}.  The conclusions are similar to those of Theorem \ref{thm:information-lowerbound} except that the attainability factor tends to be smaller.  This is because the denominator depends on attainability of cells rather than individual outcomes.  This factor can be interpreted as the number of safe cells needed for a sufficient contingency plan.

While we leave other forms of generalization for future work, understanding implications of the extreme form assumed by Theorem \ref{thm:information-lowerbound-generalization} can offer insight into the role of generalization, as we will discuss in the next section.

\subsection{Example: Coding Agent}

The lower bound of Theorem \ref{thm:information-lowerbound} grows with the number of bits required to identify a safe outcome times the number of safe outcomes needed for a sufficient contingency plan.  Let us now interpret the scale of these quantities in the context of our coding agent example.

Each outcome $o \in \observations$ is a trajectory of future observations, such as prompts requesting code behaviors, ingested data, and logs from code testing and usage.  Specifying a single outcome might require, say, terabytes. Although that is a large amount, it is not beyond reach.  But each specific outcome is highly unlikely to be attainable.  This necessitates an elaborate contingency plan that lists a huge number -- perhaps millions or billions -- of outcomes.  Specifying millions or billions of terabytes becomes prohibitive.

Our preceding interpretation is motivated by Theorem \ref{thm:information-lowerbound}, which assumes that rewards are iid across outcomes.  Does specification of a safe proxy reward function become manageable when the agent can generalize across outcomes?  As we now explain, generalization is unlikely to overcome the challenge.

Consider generalization of the sort treated by Theorem \ref{thm:information-lowerbound-generalization}.  That lower bound is again a product of KL and attainability factors.  The KL factor is of the same form as before.  The attainability factor depends on attainability of cells rather than individual outcomes.  So the question is whether this attainability multiplier would still be very large if the agent can generalize to a practical extent.

To answer this question, we focus on the special case where environment dynamics are deterministic.  In other words, for each $\pi \in \Pi$, $\rho^*(\cdot|\pi)$ is one-hot.  It is natural to expect that similar outcomes are nearly as attainable.  An extreme form of this assumption would be that, for outcomes $o, o' \in \tilde{o}$ within the same cell, $\alpha^*(o) = \alpha^*(o')$.  In this case,
$$\sup_{\tilde{o} \in \tilde{\observations}} p_{\mathrm{att}}(\tilde{o}) = \sup_{\tilde{o} \in \tilde{\observations}} \E\left[\sup_{\pi \in \Pi} \sum_{o \in \tilde{o}} \rho^*(o|\pi)\right] = \sup_{\tilde{o} \in \tilde{\observations}} \E\left[\sup_{o \in \tilde{o}} \alpha^*(o)\right] = \sup_{o \in \observations} p_{\mathrm{att}}(o).$$
It follows that, in this case, the attainability factor for a cell is equal to that for a specific outcome.  Hence, if  one is in the millions or billions, so is the other.

\section{Limiting Capabilities}\label{sec:limiting-capabilities}

Our preceding result indicates that, unless the number of bits $\I(r^*; \hat{r})$ that specify the proxy reward function is astronomical, a sufficiently capable agent that aims to maximize expected proxy reward will induce catastrophe.  The agent is better off applying an uninformed policy to attain the contemporary level of performance $V_0$.  This result formalizes \textit{Goodhart's Law}: when a measure becomes a target, it ceases to be a good measure.

Now suppose that we constrain communication to a modest number of bits $K$.  Is there a sense in which the proxy reward can be useful to the agent even though $\I(r^*; \hat{r}) \leq K$?  In this section, we show that, if we sufficiently constrain agent capabilities, the agent can attain performance exceeding the contemporary level $V_0$.  We refer to this observation as \textit{Goldilocks' Law}.

To simplify arguments in this section, we assume that $\Pi$ and $\observations$ are finite. Define
\[
\Pi_0 = \argmax_{\pi \in \Pi} 
\E\!\left[\sum_{o \in \observations} \rho^*(o|\pi) r^*(o)\right].
\]
We refer to $\Pi_0$ as the set of \textit{optimal uninformed policies}. Let $P_0 \in \Delta_\Pi$ denote the uniform distribution over $\Pi_0$.

We now define a regularized version of $\hat \pi$. For any $\lambda>0$, let 
\begin{align*}
\hat{P}_{\lambda} \in \argmax_{P \in \Delta_\Pi} \left(\lambda \sum_{\pi \in \Pi} P(\pi) \sum_{o \in \observations} \rho^*(o|\pi) \hat{r}(o) - \KL(P \| P_0) \right).
\end{align*}
The parameter $\lambda$ can be interpreted as optimization pressure, while the resulting distribution $\hat{P}_{\lambda}$ can be interpreted as a (softly) constrained version of $\hat\pi$.  In particular, the soft constraint keeps the distribution close to $P_0$. Reducing optimization pressure corresponds to tightening constraints. The value generated by $\hat P_\lambda$ is 
\[
\hat V_\lambda
=
\E\!\left[
\sum_{\pi \in \Pi}
\hat P_\lambda(\pi) \sum_{o \in \observations} \rho^*(o|\pi) r^*(o)
\right].
\]
The following result establishes that for any number of bits communicated about the true reward function $r^*$, there exists a proxy $\hat r$ such that, if capabilities are sufficiently constrained, performance exceeds the contemporary level.
\begin{restatable}{theorem}{thmlimitingcapabilities}\label{thm:limiting-capabilities}
{\bf (Goldilocks' Law)}
For all $K> 0$, if
\begin{equation}
\label{as:uninformed-policy-variation}
\Pr\left(\min_{\pi \in \Pi_0 }\sum_{o \in \observations} \rho^*(o|\pi) r^*(o) < \max_{\pi \in \Pi_0}\sum_{o \in \observations} \rho^*(o|\pi) r^*(o)\right) > 0,
\end{equation}
then there exists $\lambda > 0$ and $\hat r$ such that
$$\I(r^*; \hat{r}) \leq K \qquad \text{and} \qquad \hat V_\lambda > V_0.$$
\end{restatable}

The condition (\ref{as:uninformed-policy-variation}) ensures that with positive probability some optimal uninformed policies are better than others when assessed with respect to the true environment $\rho^*$ and reward function $r^*$.

\begin{figure}[htbp]
\centering
\includegraphics[width=0.7\columnwidth]{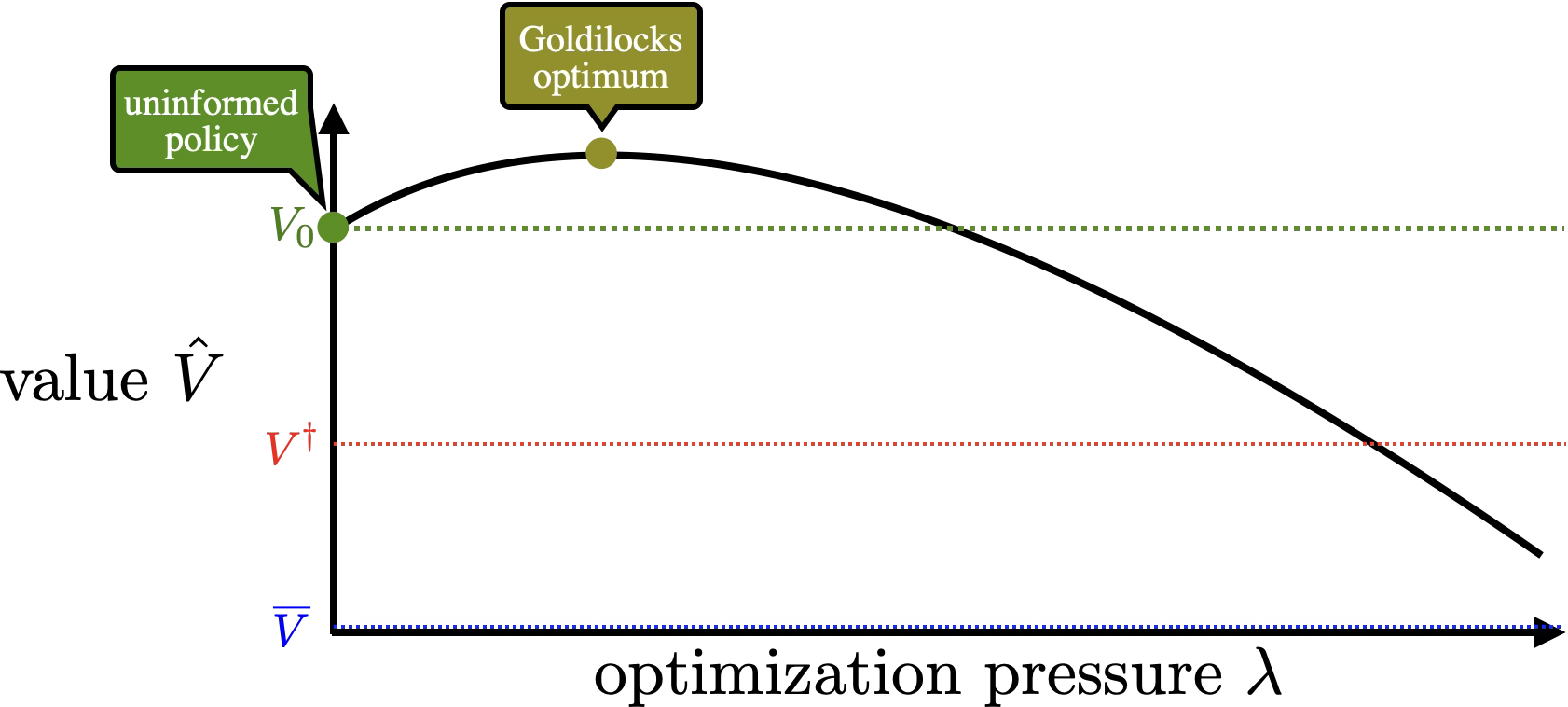}
\caption{Goldilocks' law: while a sufficiently capable agent induces catastrophe, 
if we sufficiently constrain agent capabilities, the agent can attain performance exceeding the contemporary level.}
\label{fig:goldilocks}
\end{figure}

\section{Mitigations}
\label{sec:mitigations}

We see at least two approaches for mitigating the risks discussed in this paper. 

\textbf{Constraining Capabilities.} As discussed in \Cref{sec:limiting-capabilities}, and also by others (see, for example, \citet{pan2022effectsrewardmisspecificationmapping}), constraining capabilities can mitigate the risks arising from a misspecified reward function. In practical AI systems, this is often achieved through one of two regularization techniques:
\begin{enumerate}
    \item Regularizing the policy to remain close to a base policy. For example, large language models (LLMs) are often regularized toward their pretrained models during post-training RL \citep{laidlaw2024correlated}.
    \item Early stopping, where the agent is allowed to learn from the environment only for a limited period of time \citep{gao2022scalinglawsrewardmodel}.
\end{enumerate}

Both approaches can be interpreted as limiting the amount of information the agent acquires about the environment.

However, constraining capabilities also limits the potential value an AI system can produce. Finding mitigation strategies that reduce risk without sacrificing capability is therefore an important open problem. One candidate approach is online learning of human preferences.

\textbf{Online Learning of Human Preferences.} Our theory is developed for a setting where the reward function is specified once up front, before the agent learns about the environment. The reward function may be specified manually via code, or learned through data, or some combination thereof. However, being restricted to specifying a proxy reward function up front is clearly not optimal. The agent should be able to do better by continuously updating the proxy reward function as it interacts with the designer and the environment. This reduces the burden to specify a good proxy reward function up front in three main ways:
\begin{enumerate}
    \item the designer can provide feedback to the agent as it observes the agent acting in the environment, 
    \item the agent can actively query the designer about outcomes the agent is uncertain about,
    \item the agent may learn directly from the environment about what the designer wants.
\end{enumerate}

We discuss some approaches to continual reward learning in the related work section.

\section{Related Work}
\textbf{Misalignment risk}. There is a large body of work, both academic and non-academic, discussing potential risks from advanced AI systems whose objectives are misaligned with human preferences. Early influential discussions include \citet{bostrom2012superintelligent, bostrom2014superintelligence,omohundro2018basic}. Subsequent work has analyzed mechanisms through which misaligned systems may cause harm \citep{amodei2016concrete,barnett2024without}, and the potential for advanced AI systems to seek and accumulate power \citep{NEURIPS2021_c26820b8,carlsmith2021power}. Broader discussions of AI risk and alignment also appear in \citep{russell2019human,yudkowsky2025if}.

\textbf{Reward hacking.} Many examples demonstrate agents exploiting flaws in reward functions, typically producing benign but unintended behavior \citep{krakovna2019specification,leike2017gridworlds,nguyen2015deep}. Some works formalize reward hacking and reward misspecification. For example, \citet{skalse2022defining} and \citet{laidlaw2024correlated}  provide a formal definition of reward hacking. The works of \citet{pan2022effectsrewardmisspecificationmapping,laidlaw2024correlated} analyze how misspecified reward functions lead to proxy optimization failures. Closely related is the literature on Goodhart’s law, which studies how optimizing proxy objectives can diverge from intended goals \citep{manheim2019categorizingvariantsgoodhartslaw,elmhamdi2024goodhartslawapplicationvalue,strathern1997improving}.  %

\textbf{Limiting capabilities.} Several works discuss limiting optimization power as a way of mitigating risks from misspecified objectives \citep{zhuang2020consequences,karwowski2023goodhart,taylor2016quantilizers}. The work by \citep{gao2022scalinglawsrewardmodel} studies the effects of increasing optimization pressure on proxy objectives.

\textbf{Reward learning.} A large literature studies methods for learning reward functions rather than manually specifying them. Approaches include learning reward models from human comparisons \citep{ibarz2018reward}, and active learning approaches \citep{hadfield2016cooperative,sadigh2017active,NIPS2012_16c222aa}. Scalable oversight techniques such as debate \citep{irving2018ai} and reward modeling \citep{leike2018scalable} aim to enable reliable supervision of increasingly capable systems.

\textbf{Economic principal--agent models.} The problem of specifying reward functions in AI is closely related to the principal--agent problem in economics \citep{hadfield2019incomplete}, where a principal designs incentives for an agent whose actions are imperfectly observable \citep{holmstrom1991multitask,laffont2001incentive}. These models highlight how optimizing measurable proxies can lead agents to pursue outcomes that diverge from the principal's true objectives.

\section{Conclusion}
In this work, we give conditions under which virtually any consequentialist objective induces catastrophic outcomes, if the agent is sufficiently capable. Specifically, to specify a safe objective, the AI designer would have to provide a prohibitively large amount of information. We also established conditions under which limiting capabilities not only averts catastrophic performance, but can enable the agent to produce valuable outcomes.

We highlight two research directions that we find particularly important. In practice, AI systems are often trained using reinforcement learning with early stopping and are then deployed with frozen parameters. Such systems are not explicitly optimizing a consequentialist objective at deployment time. Nevertheless, some worry that these systems may implicitly behave as if they are optimizing one, potentially posing catastrophic risks. Future work should therefore establish conditions for when frozen AI systems are safe and not safe.

Finally, as discussed in Section \ref{sec:mitigations}, we hope future work studies the conditions under which catastrophic performance can and cannot be averted through continual updating of the objective.

\section*{Acknowledgements}

This work was supported by grants from the Huxley Foundation and the UK AI Security Institute.  We thank John Duchi, Ramesh Johari, Stuart Russell, Renyuan Xu, Wanqiao Xu, and Yifan Zhu for helpful comments.

\clearpage

\bibliography{references}

@article{bostrom2012superintelligent,
  title={The superintelligent will: Motivation and instrumental rationality in advanced artificial agents},
  author={Bostrom, Nick},
  journal={Minds and Machines},
  volume={22},
  number={2},
  pages={71--85},
  year={2012},
  publisher={Springer}
}

@book{bostrom2014superintelligence,
  title={Superintelligence: Paths, Dangers, Strategies},
  author={Bostrom, Nick},
  year={2014},
  publisher={Oxford University Press}
}

@book{russell2019human,
  title={Human compatible: {AI} and the problem of control},
  author={Russell, Stuart},
  year={2019},
  publisher={Penguin UK}
}

@article{amodei2016concrete,
  title={Concrete Problems in {AI} Safety},
  author={Amodei, Dario and Olah, Chris and Steinhardt, Jacob and Christiano, Paul and Schulman, John and Mané, Dan},
  journal={arXiv preprint arXiv:1606.06565},
  year={2016}
}

@techreport{carlsmith2021power,
  title={Is Power-Seeking {AI} an Existential Risk?},
  author={Carlsmith, Joe},
  institution={Open Philanthropy},
  year={2021}
}

@inproceedings{NEURIPS2021_c26820b8,
 author = {Turner, Alex and Smith, Logan and Shah, Rohin and Critch, Andrew and Tadepalli, Prasad},
 booktitle = {Advances in Neural Information Processing Systems},
 editor = {M. Ranzato and A. Beygelzimer and Y. Dauphin and P.S. Liang and J. Wortman Vaughan},
 pages = {23063--23074},
 publisher = {Curran Associates, Inc.},
 title = {Optimal Policies Tend To Seek Power},
 url = {https://proceedings.neurips.cc/paper_files/paper/2021/file/c26820b8a4c1b3c2aa868d6d57e14a79-Paper.pdf},
 volume = {34},
 year = {2021}
}

@article{barnett2024without,
  title={WITHOUT FUNDAMENTAL ADVANCES, MISALIGNMENT AND CATASTROPHE ARE THE DEFAULT OUTCOMES OF TRAINING POWERFUL {AI}},
  author={Barnett, Peter and Gillen, Jeremy},
  year={2024}
}

@book{yudkowsky2025if,
  title={If Anyone Builds It, Everyone Dies: Why Superhuman AI Would Kill Us All},
  author={Yudkowsky, Eliezer and Soares, Nate},
  year={2025},
  publisher={Little, Brown and Company},
  isbn={978-0316595643},
}

@incollection{omohundro2018basic,
  title={The basic {AI} drives},
  author={Omohundro, Stephen M},
  booktitle={Artificial intelligence safety and security},
  pages={47--55},
  year={2018},
  publisher={Chapman and Hall/CRC}
}

@article{skalse2022defining,
  title={Defining and characterizing reward gaming},
  author={Skalse, Joar and Howe, Nikolaus and Krasheninnikov, Dmitrii and Krueger, David},
  journal={Advances in Neural Information Processing Systems},
  volume={35},
  pages={9460--9471},
  year={2022}
}

@misc{pan2022effectsrewardmisspecificationmapping,
  title={The Effects of Reward Misspecification: Mapping and Mitigating Misaligned Models},
  author={Pan, Alexander and Bhatia, Kush and Steinhardt, Jacob},
  year={2022},
  eprint={2201.03544},
  archivePrefix={arXiv},
  primaryClass={cs.LG},
  url={https://arxiv.org/abs/2201.03544}
}

@article{laidlaw2024correlated,
  title={Correlated proxies: A new definition and improved mitigation for reward hacking},
  author={Laidlaw, Cassidy and Singhal, Shivam and Dragan, Anca},
  journal={arXiv preprint arXiv:2403.03185},
  year={2024}
}

@article{krakovna2019specification,
  title={Specification Gaming: The Flip Side of {AI} Ingenuity},
  author={Krakovna, Viktoria and Uesato, Jonathan and Mikulik, Vladimir and Rahtz, Matthew and Everitt, Tom and Kumar, Ramana and Kenton, Zachary and Legg, Shane},
  journal={arXiv preprint arXiv:1909.07563},
  year={2019}
}

@article{leike2017gridworlds,
  title={{AI} Safety Gridworlds},
  author={Leike, Jan and Martic, Miljan and Krakovna, Viktoria and Ortega, Pedro and Everitt, Tom and Lefrancq, Andrew and Orseau, Laurent and Legg, Shane},
  journal={arXiv preprint arXiv:1711.09883},
  year={2017}
}

@inproceedings{
zhong2026impossiblebench,
title={ImpossibleBench: Measuring {LLM}s' Propensity of Exploiting Test Cases},
author={Ziqian Zhong and Aditi Raghunathan and Nicholas Carlini},
booktitle={The Fourteenth International Conference on Learning Representations},
year={2026},
url={https://openreview.net/forum?id=SeO4vyAj7E}
}

@incollection{bostrom2003ethical,
  author    = {Nick Bostrom},
  title     = {Ethical Issues in Advanced Artificial Intelligence},
  booktitle = {Cognitive, Emotive and Ethical Aspects of Decision Making in Humans and in Artificial Intelligence},
  editor    = {Iva Smit and George E. Lasker},
  volume    = {2},
  pages     = {12--17},
  year      = {2003},
  publisher = {International Institute of Advanced Studies in Systems Research and Cybernetics},
  address   = {Windsor, ON, Canada},
  url       = {https://nickbostrom.com/ethics/ai},
  urldate   = {2026-03-23}
}

@misc{manheim2019categorizingvariantsgoodhartslaw,
  title={Categorizing Variants of {G}oodhart's Law},
  author={David Manheim and Scott Garrabrant},
  year={2019},
  eprint={1803.04585},
  archivePrefix={arXiv},
  primaryClass={cs.AI},
  url={https://arxiv.org/abs/1803.04585}
}

@misc{elmhamdi2024goodhartslawapplicationvalue,
  title={On {G}oodhart's law, with an application to value alignment},
  author={El-Mahdi El-Mhamdi and Lê-Nguyên Hoang},
  year={2024},
  eprint={2410.09638},
  archivePrefix={arXiv},
  primaryClass={stat.ML},
  url={https://arxiv.org/abs/2410.09638}
}

@article{strathern1997improving,
  title={‘{I}mproving ratings’: audit in the {British} University system},
  author={Strathern, Marilyn},
  journal={European review},
  volume={5},
  number={3},
  pages={305--321},
  year={1997},
  publisher={Cambridge University Press}
}

@inproceedings{di2022goal,
  title={Goal misgeneralization in deep reinforcement learning},
  author={Di Langosco, Lauro Langosco and Koch, Jack and Sharkey, Lee D and Pfau, Jacob and Krueger, David},
  booktitle={International Conference on Machine Learning},
  pages={12004--12019},
  year={2022},
  organization={PMLR}
}

@article{hubinger2019risks,
  title={Risks from learned optimization in advanced machine learning systems},
  author={Hubinger, Evan and van Merwijk, Chris and Mikulik, Vladimir and Skalse, Joar and Garrabrant, Scott},
  journal={arXiv preprint arXiv:1906.01820},
  year={2019}
}

@article{ibarz2018reward,
  title={Reward learning from human preferences and demonstrations in {A}tari},
  author={Ibarz, Borja and Leike, Jan and Pohlen, Tobias and Irving, Geoffrey and Legg, Shane and Amodei, Dario},
  journal={Advances in neural information processing systems},
  volume={31},
  year={2018}
}

@article{hadfield2016cooperative,
  title={Cooperative inverse reinforcement learning},
  author={Hadfield-Menell, Dylan and Russell, Stuart J and Abbeel, Pieter and Dragan, Anca},
  journal={Advances in neural information processing systems},
  volume={29},
  year={2016}
}

@misc{gao2022scalinglawsrewardmodel,
  title={Scaling Laws for Reward Model Overoptimization},
  author={Leo Gao and John Schulman and Jacob Hilton},
  year={2022},
  eprint={2210.10760},
  archivePrefix={arXiv},
  primaryClass={cs.LG},
  url={https://arxiv.org/abs/2210.10760}
}

@inproceedings{taylor2016quantilizers,
  title={Quantilizers: A Safer Alternative to Maximizers for Limited Optimization},
  author={Taylor, Jessica},
  booktitle={Workshops of the Thirtieth AAAI Conference on Artificial Intelligence: AI, Ethics, and Society},
  year={2016},
  pages={1--9},
  url={https://intelligence.org/files/QuantilizersSaferAlternative.pdf}
}

@inproceedings{nguyen2015deep,
  title={Deep neural networks are easily fooled: High confidence predictions for unrecognizable images},
  author={Nguyen, Anh and Yosinski, Jason and Clune, Jeff},
  booktitle={Proceedings of the IEEE Conference on Computer Vision and Pattern Recognition},
  pages={427--436},
  year={2015}
}

@article{zhuang2020consequences,
  title={Consequences of misaligned {AI}},
  author={Zhuang, Simon and Hadfield-Menell, Dylan},
  journal={Advances in Neural Information Processing Systems},
  volume={33},
  pages={15763--15773},
  year={2020}
}

@article{karwowski2023goodhart,
  title={Goodhart's law in reinforcement learning},
  author={Karwowski, Jacek and Hayman, Oliver and Bai, Xingjian and Kiendlhofer, Klaus and Griffin, Charlie and Skalse, Joar},
  journal={arXiv preprint arXiv:2310.09144},
  year={2023}
}

@article{irving2018ai,
  title={{AI} safety via debate},
  author={Irving, Geoffrey and Christiano, Paul and Amodei, Dario},
  journal={arXiv preprint arXiv:1805.00899},
  year={2018}
}

@article{leike2018scalable,
  title={Scalable agent alignment via reward modeling: a research direction},
  author={Leike, Jan and Krueger, David and Everitt, Tom and Martic, Miljan and Maini, Vishal and Legg, Shane},
  journal={arXiv preprint arXiv:1811.07871},
  year={2018}
}

@article{sadigh2017active,
  title={Active preference-based learning of reward functions},
  author={Sadigh, Dorsa and Dragan, Anca and Sastry, Shankar and Seshia, Sanjit},
  year={2017}
}

@inproceedings{NIPS2012_16c222aa,
 author = {Wilson, Aaron and Fern, Alan and Tadepalli, Prasad},
 booktitle = {Advances in Neural Information Processing Systems},
 editor = {F. Pereira and C.J. Burges and L. Bottou and K.Q. Weinberger},
 pages = {},
 publisher = {Curran Associates, Inc.},
 title = {A Bayesian Approach for Policy Learning from Trajectory Preference Queries},
 url = {https://proceedings.neurips.cc/paper_files/paper/2012/file/16c222aa19898e5058938167c8ab6c57-Paper.pdf},
 volume = {25},
 year = {2012}
}

@article{holmstrom1991multitask,
  title={Multitask principal--agent analyses: Incentive contracts, asset ownership, and job design},
  author={Holmstrom, Bengt and Milgrom, Paul},
  journal={The Journal of Law, Economics, and Organization},
  volume={7},
  number={special\_issue},
  pages={24--52},
  year={1991},
  publisher={Oxford University Press}
}

@article{laffont2001incentive,
  title={The Theory of Incentives: The Principal-Agent Model},
  author={Laffont, Jean-Jacques and Martimort, David},
  year={2001},
  publisher={Princeton University Press},
  address={Princeton, NJ},
  isbn={9780691091150}
}

@misc{aioptimism2023easycontrol,
  title        = {{AI} is Easy to Control},
  author       = {Belrose, Nora and Pope, Quintin.},
  year         = {2023},
  month        = nov,
  day          = {28},
  url          = {https://optimists.ai/2023/11/28/ai-is-easy-to-control/},
  note         = {Accessed: 2026-03-08}
}

@online{mitchell2024dumbsuperintelligence,
  author       = {Melanie Mitchell},
  title        = {Debunking the ``dumb superintelligence'' myth},
  year         = {2024},
  organization = {Freethink},
  url          = {https://www.youtube.com/watch?v=oR8E9dVsMVs},
  note         = {YouTube video}
}

@article{carone2019dontfearterminator,
  title   = {Don't Fear the Terminator},
  author  = {Carone, Timothy E.},
  journal = {Scientific American: Observations Blog},
  year    = {2019},
  month   = {September},
  url     = {https://www.scientificamerican.com/blog/observations/dont-fear-the-terminator/},
  note    = {Accessed: 2026-03-09}
}

@misc{pace2019instrumentalconvergence,
  title        = {Debate on Instrumental Convergence between {LeCun}, {Russell}, {Bengio}, {Zador}, and More},
  author       = {Pace, Ben},
  year         = {2019},
  howpublished = {\url{https://www.alignmentforum.org/posts/WxW6Gc6f2z3mzmqKs/debate-on-instrumental-convergence-between-lecun-russell}},
  note         = {AI Alignment Forum, Accessed: 2026-03-09}
}

@online{munk2023aidebate,
  title        = {Munk Debate on Artificial Intelligence | {Bengio} \& {Tegmark} vs. {Mitchell} \& {LeCun}},
  author       = {Munk-Debates},
  year         = {2023},
  month        = {June},
  url          = {https://www.youtube.com/watch?v=144uOfr4SYA},
  note         = {YouTube video}
}

@inproceedings{hadfield2019incomplete,
  title={Incomplete contracting and {AI} alignment},
  author={Hadfield-Menell, Dylan and Hadfield, Gillian K},
  booktitle={Proceedings of the 2019 AAAI/ACM Conference on AI, Ethics, and Society},
  pages={417--422},
  year={2019}
}

@article{alon2002lower,
  title={A lower bound on the expected length of one-to-one codes},
  author={Alon, Noga and Orlitsky, Alon},
  journal={IEEE Transactions on Information Theory},
  volume={40},
  number={5},
  pages={1670--1672},
  year={2002},
  publisher={IEEE}
}

\clearpage
\appendix

\section{Proofs}\label{apx:proofs}
\subsection{Some Useful Mathematical Facts}
Each lemma in this section is a stand-alone mathematical result, framed abstractly in a manner not directly related to the subject of this paper.  We will use these lemmas in our subsequent analysis.

\begin{lemma}
\label{le:kl-decomposition-product-measure}
{\bf (KL decomposition)}
For all real-valued finite random sequences $X=(X_1,X_2, \ldots, X_T)$ and $Y=(Y_1,Y_2, \ldots, Y_T)$, if the elements of $Y$ are independent then
\begin{equation*}
\KL(\Pr(X \in \cdot) || \Pr(Y \in \cdot)) = \sum_{t=1}^T \left(\KL(\Pr(X_t \in \cdot) || \Pr(Y_t \in \cdot)) + \I(X_t ; X_{1:t-1})\right).
\end{equation*}
\end{lemma}
\begin{proof}
\begin{align*}
\KL(\Pr(X \in \cdot) || \Pr(Y \in \cdot)) 
=& \int_{z \in \mathcal{X}} \Pr(X \in dz) \log \frac{\Pr(X \in dz)}{\Pr(Y \in dz)} \\
=& \int_{z \in \mathcal{X}}  \Pr(X \in dz)  \log \prod_{t=1}^T \frac{\Pr(X_t \in dz_t | X_{1:t-1} \in dz_{1:t-1})}{\Pr(Y_t \in dz_t | Y_{1:t-1} \in dz_{1:t-1})} \\
=& \sum_{t=1}^T \int_{z \in \mathcal{X}}  \Pr(X \in dz)  \log \frac{\Pr(X_t \in dz_t | X_{1:t-1} \in dz_{1:t-1})}{\Pr(Y_t \in dz_t | Y_{1:t-1} \in dz_{1:t-1})} \\
=& \sum_{t=1}^T \int_{z \in \mathcal{X}}  \Pr(X \in dz)  \log \frac{\Pr(X_t \in dz_t | X_{1:t-1} \in dz_{1:t-1})}{\Pr(Y_t \in dz_t)} \\
\stackrel{(a)}{=}& \sum_{t=1}^T \left(\int_{z \in \mathcal{X}}  \Pr(X \in dz) \log 
\frac{\Pr(X_t \in dz_t) \Pr(X_{1:t-1}\in dz_{1:t-1}|X_t \in dz_t)}{\Pr(Y_t \in dz_t) \Pr(X_{1:t-1} \in dz_{1:t-1})}\right) \\
=& \sum_{t=1}^T \int_{z \in \mathcal{X}}  \Pr(X \in dz) \left(\log \frac{\Pr(X_t \in dz_t)}{\Pr(Y_t \in dz_t)} + \log \frac{\Pr(X_{1:t-1}\in dz_{1:t-1}|X_t\in dz_t)}{\Pr(X_{1:t-1}\in dz_{1:t-1})}\right) \\
=& \sum_{t=1}^T (\KL(\Pr(X_t \in \cdot) \| \Pr(Y_t \in \cdot)) 
+ \E[\KL(\Pr(X_{1:t-1} \in \cdot|X_t) \| \Pr(X_{1:t-1}\in \cdot)]) \\
\stackrel{(b)}{=}& \sum_{t=1}^T (\KL(\Pr(X_t \in \cdot) \| \Pr(Y_t \in \cdot)) + \I(X_t; X_{1:t-1})).
\end{align*}
where (a) follows from Bayes' rule and (b) follows from the expected KL-divergence expression of mutual information.
\end{proof}

We refer to a real-valued sequence $(w_t: t = 1,2,\ldots)$ as a {\it probability sequence} if the elements are nonnegative and sum to one.  We denote the set of probability sequences by $\Delta_\infty$.
\begin{lemma}
{\bf (frontloading)}
\label{le:frontloading}
For all $q \in [0,1]$, probability sequences $(w_t: t=1,2\ldots)$, and nonincreasing sequences $(a_t: t=1,2\ldots)$, if $\sup_{t\in \mathbb{Z}_{++}} w_t \le q$ then
\begin{equation}
\sum_{t=1}^\infty w_t a_t \le q \sum_{t=1}^T a_t + (1 - q T) a_{T+1},\label{eq:frontloading}
\end{equation}
where $T = \lfloor 1/q \rfloor$.
\end{lemma}
\begin{proof}
Let $\mathcal{W}$ be the set of probability sequences $w=(w_t:t=1,2,...)$ for which $\sup_t w_t \leq q$. 
Note that the left hand side of \eqref{eq:frontloading} is an expectation of a random variable $A$ for which $\Pr(A=a_t) = w_t$. Similarly, the right hand side is an expectation of a random variable $A'$ for which $\Pr(A'=a_t) = w'_t$, where $w'_1=\cdots=w'_T = q$ and $w'_{T+1} = 1 - q T$.  Note that, for all $w \in \mathcal W$ and $t\in\mathbb Z^+$, 
$$\sum_{s=1}^t w_s \leq \sum_{s=1}^t w'_s,$$ 
since $w_s \leq q = w'_s$ for $s \leq T$ and $\sum_{s=1}^t w_s\leq 1 =\sum_{s=1}^t w'_s$ for $t \geq T+1$. Since $a_t$ is non-increasing, for all $w\in \mathcal W$, this implies $A'$ first-order stochastically dominates $A$. It follows that $\E[A']\geq \E[A]$.
\end{proof}

\begin{lemma}
\label{lm:bound-expectation-of-positive-component}
    Let $X$ be a real-valued random variable with mean $\mu$ and variance $\sigma^2 < \infty$. Let $X_+ = \max(X,0)$. Then,
    $$
        \E[X_+] \le \frac{\mu + \sqrt{\mu^2 + \sigma^2}}{2}
    $$
\end{lemma}
\begin{proof}
    \begin{align*}
        \E[X_+] = \E\left[\frac{X + |X|}{2}\right] 
                = \frac{\E[X] + \E[|X|]}{2} \stackrel{(a)}{\le}  
                \frac{\E[X] + \sqrt{\E[X^2]}}{2} = \frac{\mu + \sqrt{\mu^2 + \sigma^2}}{2}
    \end{align*}
Step (a) follows from Cauchy-Schwarz.
\end{proof}

\begin{lemma}
\label{lm:bound-expectation-of-positive-component-simplified}
  Let $X$ be a real-valued random variable with mean $\mu \ge 0$ and variance $\sigma^2$. Let $X_+ = \max(X, 0)$. Then,
    $$
        \E[X_+] \le \mu + \sigma/2
    $$
\end{lemma}
\begin{proof}
    \begin{align*}
        \E[X_+] \stackrel{(a)}{\le} \frac{\mu + \sqrt{\mu^2 + \sigma^2} }{2}  = \frac{\mu + \sqrt{(\mu + \sigma)^2-2\mu\sigma} }{2}\le \frac{\mu + \sqrt{(\mu + \sigma)^2} }{2} \stackrel{(b)}{=} \frac{\mu + \mu + \sigma}{2} = \mu + \sigma/2.
    \end{align*}
where (a) follows from Lemma \ref{lm:bound-expectation-of-positive-component} and (b) follows due to $\mu, \sigma \ge 0$.
\end{proof}

\subsection{Proofs for \Cref{sec:many-bits}}

In this section, we will make use of the notion of a coherent reward function. 

\begin{restatable}{definition}{coherencedef} (Coherent reward function). We say a random reward function $\check{r}$ is coherent if
$$
\check{r} = \E\left[r^*|\check{r}\right]. 
$$
\end{restatable}
We will first prove a result for the case when $\hat r$ is coherent (\Cref{thm:semi-general-coherent}). Then we will consider the general case. 

For a reward function $r\in\mathcal R$, let $o_{r,t} \in \argmax_{o \in \observations \setminus \{o_{r,1},\ldots,o_{r,t-1}\}} r(o)$, with ties broken by selecting the observation with the smallest index. Let $p_{\mathrm{att}}' =\sup_{o\in\mathcal O} p_{\mathrm{att}}(o)$. 

\begin{lemma}
\label{le:dominating-performance}
If $\hat{r}$ is coherent and $r^*\perp \mathcal L_{\rho^*}$ then
$$\hat{V} \le \E\left[\sum_{t=1}^T p_{\mathrm{att}}' r^*(o_{\hat{r}, t}) + (1 - p_{\mathrm{att}}' T)r^*(o_{\hat{r}, T+1})\right],$$
where $T = \lfloor 1/{p_{\mathrm{att}}'} \rfloor$.
\end{lemma}
\begin{proof}  
For each $r \in \mathcal{R}$ and $t$, let $w_{r,t} = \E[\rho^*(o_{r,t}|\pi_{\rho^*, r})]$, and note that $(w_{r,1}, w_{r,2}, \ldots)$ is a probability sequence that is uniformly bounded by $p_{\mathrm{att}}'$. Further, note that $\rho^*(o|\pi_{\rho^*, \hat r})$ is $(\mathcal {L}_{\rho^*},\hat r)$-measurable. 

We have
\begin{align*}
\hat{V}
&=\E\left[\sum_{o\in\observations }\rho^*(o|\pi_{\rho^*,\hat r})r^*(o)\right] \\
&=\E\left[\sum_{t=1}^\infty \rho^*(o_{\hat r,t}|\pi_{\rho^*,\hat r})r^*(o_{\hat r,t})\right] \\
&=\E\left[\sum_{t=1}^\infty \E \left[\E \left[\rho^*(o_{\hat r,t}|\pi_{\rho^*,\hat r})r^*(o_{\hat r,t})\bigr | \hat r, \mathcal L_{\rho^*}\right]\biggr| \hat r\right]\right] \\
&=\E\left[\sum_{t=1}^\infty \E \left[\rho^*(o_{\hat r,t}|\pi_{\rho^*,\hat r})\E \left[r^*(o_{\hat r,t})\bigr | \hat r, \mathcal L_{\rho^*}\right]\biggr| \hat r\right]\right] \\
&\stackrel{(a)}{=} \E\left[\sum_{t=1}^\infty \E[\rho^*(o_{\hat r,t}|\pi_{\rho^*,\hat r})|\hat{r}] \E[r^*(o_{\hat{r},t})|\hat{r}] \right] \\
&\stackrel{(b)}{=} \E\left[\sum_{t=1}^\infty w_{\hat{r},t} \hat{r}(o_{\hat{r},t})\right] \\
&\stackrel{(c)}{\leq} \E\left[\sum_{t=1}^T p_{\mathrm{att}}' \hat{r}(o_{\hat{r},t}) + (1-p_{\mathrm{att}}' T) \hat{r}(o_{\hat{r},T+1})\right] \\
&\stackrel{(d)}{=} \E\left[\sum_{t=1}^T p_{\mathrm{att}}' r^*(o_{\hat{r},t}) + (1-p_{\mathrm{att}}' T) r^*(o_{\hat{r},T+1})\right].
\end{align*}
where (a) holds because $\mathcal L_{\rho^*}\perp (r^*, \hat{r})$ (which follows since $\hat r \perp \rho^*|r^*$ and $r^*\perp \mathcal L_{\rho^*}$), (b,d) follow from coherence, and (c) follows from Lemma \ref{le:frontloading}.
\end{proof}

Let $r^*_+$ be defined for all $o \in \observations$ by $r^*_+(o) = \max \{r^*(o)  ,0\}$. Recall that $\overline O$ is the primordial observation. Let $\overline V_+ = \E[r_+^*(\overline O)]$.

\begin{theorem}\label{thm:semi-general-coherent}
If $\hat{r}$ is coherent, $r^* \perp \mathcal{L}_{\rho^*}$, $(r^*(o): o \in \observations)$ is iid, and $\hat{V} \geq \overline{V}_+$ then
\begin{align*}
\I(r^*;\hat r) &\ge \frac{1}{p_{\mathrm{att}}'} \KL(\mathrm{Bern}(\hat{V}) || \mathrm{Bern}(\overline{V}_+)).
\end{align*}
\end{theorem}

\begin{proof}
Let $\eta_t = \E[r^*_+(o_{\hat{r},t})]$.  Let $\tilde{r}$ be distributed identically with $\hat{r}$ but independent of $(r^*, \hat{r})$.

Let $(o_1, o_2,... )$ be an enumeration of $\observations$ and recall that $(r^*(o_1), r^*(o_2), \ldots)$ is an iid sequence. Since $(r^*(o_{\tilde{r},1}), r^*(o_{\tilde{r},2}), \ldots)$ is a random permutation of $(r^*(o_1), r^*(o_2), \ldots)$ for which the indices $(o_{\tilde{r},1}, o_{\tilde{r},2}, \ldots)$ are independent of $r^*$,  it follows that $(r^*(o_{\tilde{r},1}), r^*(o_{\tilde{r},2}), \ldots)$ is an iid sequence. 

Recall that $\rho(o|\pi_{\rho^*,r})$ is $\mathcal L_{\rho^*}$-measurable. Recall that $\overline r$ is the best uninformed reward function. It follows that
\begin{align*}
\overline{V}_+ &= 
\E[r_+^*(\overline O)]\\
&=\E \sum_{o\in\observations}\rho^*(o|\pi_{\rho^*,\overline r})r_+^*(o)\\
&=  \E \sum_{o\in\observations} \E \left[\rho^*(o|\pi_{\rho^*,\overline r})|\mathcal L_{\rho^*}\right]\E r_+^*(o)\\
&= (\E r_+^*(o')) \E\sum_{o\in\observations} \E \left[\rho^*(o|\pi_{\rho^*,\overline r})|\mathcal L_{\rho^*}\right]\ \ \text{for any } o'\in\observations\\
&= \E[r^*_+(o_{\tilde{r},t})]
\end{align*}
for any $t$, where we have used the fact that $\E\sum_{o\in\observations} \E \left[\rho^*(o|\pi_{\rho^*,\overline r})|\mathcal L_{\rho^*}\right]=1$ and the fact that we use consequentialist tie-breaking.

For each $o \in \observations$, let $r^b(o)$ be independently sampled from $\mathrm{Bern}(r^*_+(o))$.  Hence, $\E[r^b(o_{\hat{r},t})] = \E[r^*_+(o_{\hat{r},t})] = \eta_t$ and $\E[r^b(o_{\tilde{r},t})] = \E[r^*_+(o_{\tilde{r},t})] = \overline{V}_+$.

For any $T$,
\begin{align*}
\I(r^* ; \hat r) 
&\stackrel{(a)}{\ge} \I(r^* ; o_{\hat{r},1:T}) \\
&\stackrel{}{=} \KL\bigg(\Pr((r^*, o_{\hat{r},1:T}) \in \cdot ) || \Pr((r^*, o_{\tilde{r}, 1:T})  \in \cdot)  \bigg) \\
&\stackrel{(b)}{\ge} \KL\bigg(\Pr((r^*(o_{\hat{r}, 1}), \ldots, r^*(o_{\hat{r}, T})) \in \cdot ) || \Pr((r^*(o_{\tilde{r}, 1}), \ldots, r^*(o_{\tilde{r}, T})) \in \cdot) \bigg) \\
&\stackrel{(c)}{=} \sum_{t=1}^T (\KL(\Pr (r^*(o_{\hat{r}, t}) \in \cdot) || \Pr (r^*(o_{\tilde{r},t}) \in \cdot) ) + \I(r^*(o_{\hat{r},t}) ; r^*(o_{\hat{r},1}), \ldots, r^*(o_{\hat{r},t-1}))) \\
&\geq \sum_{t=1}^T \KL(\Pr (r^*(o_{\hat{r}, t}) \in \cdot) || \Pr (r^*(o_{\tilde{r},t}) \in \cdot) ) \\
&\stackrel{(d)}{\geq} \sum_{t=1}^T \KL(\Pr (r^b(o_{\hat{r}, t}) \in \cdot) || \Pr (r^b(o_{\tilde{r},t}) \in \cdot) ) \\
&= \sum_{t=1}^T \KL(\mathrm{Bern}(\eta_t) || \mathrm{Bern}(\overline{V}_+)),
\end{align*}
where (a,b,d) follow from the data processing inequality and (c) follows from Lemma \ref{le:kl-decomposition-product-measure} and the fact that $(r^*(o_{\tilde{r},1}), \ldots, r^*(o_{\tilde{r},T}))$ is an iid sequence.

For each $t$, let
\begin{align*}
w_t =
    \begin{cases}
    1 & \text{if } t \le T, \\
    (1-p_{\mathrm{att}}'T)/p_{\mathrm{att}}'& \text{if } t = T + 1, \\
    0 & \text{otherwise},
    \end{cases}
\end{align*}
where $T=\lfloor 1/p_{\mathrm{att}}' \rfloor$. Note that $(p_{\mathrm{att}}' w_t : t=1,2,\ldots)$ is a probability sequence.  

Note that
\begin{align*}
\hat{V} 
&\stackrel{(a)}{\le}\E\left[\sum_{t=1}^{T} p_{\mathrm{att}}' r^*(o_{\hat r, t}) + (1 - p_{\mathrm{att}}'T) r^*(o_{\hat r, T+1})\right]  \\
&\leq \sum_{t=1}^{T} p_{\mathrm{att}}' \eta_t + (1-p_{\mathrm{att}}' T)\eta_{T+1} \\
&= \sum_{t=1}^{T+1} p_{\mathrm{att}}' w_t \eta_t,
\end{align*}
where (a) follows from Lemma \ref{le:dominating-performance}.

We then have
\begin{align*}
    \sum_{t=1}^{T+1} \KL(\mathrm{Bern}(\eta_t) || \mathrm{Bern}(\overline{V}_+)) &= \frac{1}{p_{\mathrm{att}}'}  \sum_{t=1}^{T+1} p_{\mathrm{att}}' \KL(\mathrm{Bern}  (\eta_t) \| \mathrm{Bern}(\overline{V}_+)) \\
    &= \frac{1}{p_{\mathrm{att}}'}  \sum_{t=1}^{T+1} \left( p_{\mathrm{att}}'w_t \KL(\mathrm{Bern}  (\eta_t) || \mathrm{Bern}(\overline{V}_+)) +  p_{\mathrm{att}}' (1-w_t)\KL(\mathrm{Bern}  (\eta_t) \| \mathrm{Bern}(\overline{V}_+)) \right) \\
    &\ge \frac{1}{p_{\mathrm{att}}'}\sum_{t=1}^{T+1} p_{\mathrm{att}}'w_t \KL(\mathrm{Bern}(\eta_t) \| \mathrm{Bern}(\overline{V}_+)) \\
    &\stackrel{(a)}{\ge}  \frac{1}{p_{\mathrm{att}}'}\KL\left(\mathrm{Bern}\left(\sum_{t=1}^{T+1} p_{\mathrm{att}}' w_t \eta_t\right) \Big\| \mathrm{Bern}(\overline{V}_+)\right) \\
    &\ge \frac{1}{p_{\mathrm{att}}'}\KL(\mathrm{Bern}(\hat{V}) \| \mathrm{Bern}(\overline{V}_+)),
\end{align*}
where (a) follows from Jensen's inequality and convexity of KL divergence.  The result follows.
\end{proof}
\begin{lemma}\label{co:coherent-is-better} Let $r^*\perp \mathcal L_{\rho^*}$, $\hat r'=\E[r^*|\hat r]$ and 
    $$
\hat{V}' = \E\left[\sum_{o \in \observations} \rho^*(o|\pi_{\rho^*,\hat{r}'}) r^*(o)\right].
$$
Then $\hat r'$ is coherent and $\hat V' \geq \hat V$. 
\end{lemma}
\begin{proof}
    That $\hat r'$ is coherent follows from the fact that $\hat r' = \E[r^*|\hat r] = \E[r^*|\hat r']$. To see why $\hat V' \geq \hat V$, note that
    \begin{align*}
\hat{V}'
&= \E\left[\sum_{o \in \observations} \rho^*(o|\pi_{\rho^*,\hat r'}) r^*(o)\right] \\
&= \E\left[\sum_{o \in \observations} \E[\rho^*(o|\pi_{\rho^*,\hat r'}) r^*(o)| \hat r]\right] \\
&\overset{(a)}{=} \E\left[\sum_{o \in \observations} \E\left[\rho^*(o|\pi_{\rho^*,\hat{r}'})|\hat r \right]\E[r^*(o)|\hat r]\right] \\
&= \E\left[\sum_{o \in \observations} \E\left[\rho^*(o|\pi_{\rho^*,\hat{r}'})|\hat r\right] \hat r'(o)\right] \\
&= \E\left[\sum_{o \in \observations} \E\left[\rho^*(o|\pi_{\rho^*,\hat{r}'})\hat r'(o)|\hat r\right] \right] \\
&= \E\left[\sum_{o \in \observations} \rho^*(o|\pi_{\rho^*,\hat r'}) \hat{r}'(o)\right] \\
&\geq \E\left[\sum_{o \in \observations} \rho^*(o|\pi_{\rho^*,\hat{r}}) \hat{r}'(o)\right] \\
&= \E\left[\sum_{o \in \observations} \E\left[\rho^*(o|\pi_{\rho^*,\hat{r}}) \hat r'(o)|\hat r\right]\right]\\
&= \E\left[\sum_{o \in \observations} \E\left[\rho^*(o|\pi_{\rho^*,\hat{r}}) \E\left[r^*(o)|\hat r\right]\biggr|\hat r\right]\right]\\
&= \E\left[\sum_{o \in \observations} \E\left[\rho^*(o|\pi_{\rho^*,\hat{r}}) \bigr|\hat r\right]\E\left[r^*(o)|\hat r\right]\right]\\
&\overset{(b)}{=} \E\left[\sum_{o \in \observations} \E\left[\rho^*(o|\pi_{\rho^*,\hat{r}}) r^*(o)|\hat r\right]\right]\\
&= \E\left[\sum_{o \in \observations} \rho^*(o|\pi_{\rho^*, \hat r}) r^*(o)\right] \\
&= \hat{V},
\end{align*}
where in (a) we used the fact that $\rho^*(o|\pi_{\rho^*,\hat r'})\perp r^* |\hat r$ and similarly in (b) we used the fact that $\rho^*(o|\pi_{\rho^*,\hat r})\perp r^* |\hat r$. The conditional independence used in (a) follows from noting that $\mathcal L_{\rho^*}\perp r^*$ and $\hat r\perp \rho^*|r^*$ imply that $\mathcal L_{\rho^*}\perp r^*|\hat r$, and $\rho^*(o|\pi_{\rho^*,\hat r'})$ is $\mathcal L_{\rho^*}$-measurable given $\hat r$. A similar argument holds for the conditional independence used in (b). 
\end{proof}

\begin{corollary}
\label{co:semi-general-incoherent}
If $r^*\perp\mathcal L_{\rho^*}$, $\hat{V} \geq \overline{V}_+$, and $(r^*(o):o\in\observations)$ is an iid sequence, then
\begin{align*}
\I(r^*;\hat r) &\ge \frac{1}{p_{\mathrm{att}}'} \KL(\mathrm{Bern}(\hat{V}) || \mathrm{Bern}(\overline{V}_+)).
\end{align*}
\end{corollary}
\begin{proof}
Let $\hat{r}' = \E[r^* | \hat{r}]$ and 
$$
\hat{V}' = \E\left[\sum_{o \in \observations} \rho^*(o|\pi_{\rho^*,\hat{r}'}) r^*(o)\right].
$$
By the data processing inequality, $\I(r^*; \hat{r}) \geq \I(r^*; \hat{r}')$.  Further, by \Cref{co:coherent-is-better} we have that $\hat r'$ is coherent and  $\hat V' \geq \hat V$. The result follows from Theorem \ref{thm:semi-general-coherent}.
\end{proof}

\begin{corollary}
\label{co:semi-general-incoherent-with-sigma}
If $r^*\perp\mathcal L_{\rho^*}$, $\hat{V} \geq \overline{V} + \overline \sigma/2$, and $(r^*(o):o\in\observations)$ is an iid sequence, then
\begin{align*}
\I(r^*;\hat r) &\ge \frac{1}{p_{\mathrm{att}}'} \KL(\mathrm{Bern}(\hat{V}) || \mathrm{Bern}(\overline{V} + \overline \sigma/2)).
\end{align*}
\end{corollary}
\begin{proof}
   We have by Lemma \ref{lm:bound-expectation-of-positive-component-simplified} that $\overline V_+ \le \overline V + \overline \sigma / 2$. Because $\hat{V} \geq \overline{V} + \overline \sigma/2$, 
   $$
  \KL(\mathrm{Bern}(\hat{V}) || \mathrm{Bern}(\overline{V} + \overline \sigma/2)) \le \KL(\mathrm{Bern}(\hat{V}) || \mathrm{Bern}(\overline{V}_+)).
   $$
   The result follows from Corollary \ref{co:semi-general-incoherent}.
\end{proof}

\informationlowerbound*
\begin{proof}
    We apply \Cref{co:semi-general-incoherent-with-sigma} and observe that $\hat V\geq V^\dagger \geq \overline V + \overline \sigma/2$. 
\end{proof}

\subsubsection{Generalization}

\begin{theorem}\label{le:generalization-coherent-r}
If $r^* \perp \mathcal{L}_{\rho^*}$, $\hat r$ is coherent, $\hat V\geq \overline{V}_+$, and there exists a countable partition $\tilde{\observations}$ of $\observations$ such that $r^*$ is constant within each cell of $\tilde{\observations}$ and $(r^*(\tilde{o}): \tilde{o} \in \tilde{\observations})$ is iid then
\begin{align*}
\I(r^*;\hat r) &\ge \frac{1}{\sup_{\tilde{o} \in \tilde{\observations}} p_{\mathrm{att}}(\tilde{o})} \KL(\mathrm{Bern}(\hat V) || \mathrm{Bern}(\overline{V}_+)).
\end{align*}
\end{theorem}

\begin{proof}

We will reduce to the iid case. Consider a new problem where the outcome space is $\tilde \observations$, the policy set is $\Pi$, and reward function $r^*_{\text{new}}$ is
$$
r_{\text{new}}^*(\tilde o) = r^*(o),
$$
for any $o\in\tilde o$ (recall $r^*(o)$ is constant for $o\in \tilde o$). Similarly, let
$$
\hat r_{\text{new}}(\tilde o) = \hat r(o),
$$
for any $o\in \tilde o$ (since $\hat r$ is coherent, $\hat r$ is constant for $o\in \tilde o$). Finally, let 
\begin{align}
\rho_{\text{new}}^*(\tilde o|\pi) =  \sum_{o\in\tilde o} \rho^*(o|\pi)\label{eq:rho-star-new}.
\end{align}
Notice that $(r^*_{\text{new}}(\tilde o): \tilde o\in\tilde \observations)$ is an iid collection of random variables. In this iid setting, all uninformed reward functions have equal expected performance. Let $\overline r_{\text{new}}$ be an arbitrary reward function and let
$$
\overline O_{\text{new}} \sim \rho_{\text{new}}^*(\cdot|\pi_{\rho^*_{\text{new}},\overline r_{\text{new}}})
$$
be the primordial observation in this setting. Note that for any random outcome $\tilde O\in \tilde\observations$ that is independent of $r^*_{\mathrm{new}}$, we have
\begin{align*}
r^*_{\text{new}}(\overline O_{\text{new}})\overset{\mathcal D}{=}r^*_{\text{new}}(\tilde O) . 
\end{align*}
Consider, in particular, $\tilde O$ generated as $f(\overline O)$ where $f:\observations\to\tilde\observations$ is the function that maps an outcome to its cell. Then we have,
$$
r^*_{\text{new}}(\overline O_{\text{new}}) \overset{\mathcal D}{=}  r^*_{\text{new}}(f(\overline O))= r^*(\overline O),
$$
where the latter equality holds because $r^*_{\text{new}}(f(o))=r^*(o)$ for all $o\in\observations$. We conclude,  
$$
r^*_{\text{new},+}(\overline O_{\text{new}}) \overset{\mathcal D}{=} r_+^*(\overline O),
$$
where $r^*_{\text{new},+}=\max(r^*_{\text{new}},0)$. This implies
$$
\overline V_{\text{new}, + } = \E r^*_{\text{new},+}(\overline O_{\text{new}}){=} \mathbb{E}r^*_+(\overline O)= \overline V_+.
$$
Similarly, 
\begin{align*}
\hat V &= \E\sum_{o\in\observations}\rho^*(o|\pi_{\rho^*,\hat r})r^*(o)\\
&= \E\sum_{\tilde o\in\tilde \observations}\sum_{o\in\tilde o}\rho^*(o|\pi_{\rho^*,\hat r})r^*(o)\\
&= \E\sum_{\tilde o\in\tilde \observations}r^*_{\text{new}}(\tilde o)\sum_{o\in\tilde o}\rho^*(o|\pi_{\rho^*,\hat r})\\
&\overset{(a)}{=} \E\sum_{\tilde o\in\tilde \observations}r^*_{\text{new}}(\tilde o)\rho^*_{\text{new}}(\tilde o|\pi_{\rho^*_{\text{new}},\hat r_{\text{new}}})\\
&= \hat V_{\text{new}},
\end{align*}
where (a) follows from \Cref{eq:rho-star-new} and the fact that $\pi_{\rho^*,\hat r}$ is a maximizer of 
$$
\sum_{o\in\observations}\hat r(o)\rho^*(o|\pi) = \sum_{\tilde o\in \tilde \observations }\sum_{o\in\tilde o} \rho^*(o|\pi)\hat r(o) = \sum_{\tilde o \in \tilde \observations} \hat r_{\text{new}}(\tilde o)\rho^*_{\text{new}}(\tilde o|\pi). 
$$

Notice that, since the partition $\tilde O$ is known, and $r^*$ are constant along each cell, $r^*$ is a deterministic function of $r^*_{\text{new}}$ and $r^*_{\text{new}}$ is a deterministic function of $r^*$. Similarly, since $\hat r$ is coherent it is constant along each cell and so similar functions can be constructed that take $\hat r$ to $\hat r_{\text{new}}$ and $\hat r_{\text{new}}$ to $\hat r$. It follows, by using the data processing inequality in both directions, that 
$$
 \I(r^*;\hat r)= \I(r_{\text{new}}^*;\hat r_{\text{new}}). 
$$
Finally, via \Cref{co:semi-general-incoherent}, we have
$$
\I(r^*;\hat r) = \I(r_{\text{new}}^*;\hat r_{\text{new}})\geq \frac{1}{\sup_{\tilde o \in \tilde\observations} p_{\mathrm{att}}(\tilde o)} \KL(\mathrm{Bern}(\hat V)||\mathrm{Bern}(\overline V_+)). 
$$
\end{proof}

\begin{corollary} \label{co:countable-partition-V+}
If $r^* \perp \mathcal{L}_{\rho^*}$, $\hat V\geq \overline{V}_+$, and there exists a countable partition $\tilde{\observations}$ of $\observations$ such that $r^*$ is constant within each cell of $\tilde{\observations}$ and $(r^*(\tilde{o}): \tilde{o} \in \tilde{\observations})$ is iid then
\begin{align*}
\I(r^*;\hat r) &\ge \frac{1}{\sup_{\tilde{o} \in \tilde{\observations}} p_{\mathrm{att}}(\tilde{o})} \KL(\mathrm{Bern}(\hat V) || \mathrm{Bern}(\overline{V}_+)).
\end{align*}
\end{corollary}
\begin{proof}
Let $\hat{r}' = \E[r^* | \hat{r}]$ and 
$$
\hat{V}' = \E\left[\sum_{o \in \observations} \rho^*(o|\pi_{\rho^*,\hat{r}'}) r^*(o)\right].
$$ 

By the data processing inequality, $\I(r^*; \hat{r}) \geq \I(r^*; \hat{r}')$.  Further, by \Cref{co:coherent-is-better} we have that $\hat r'$ is coherent and  $\hat V' \geq \hat V$. The result follows from \Cref{le:generalization-coherent-r}.
\end{proof}

\begin{corollary}\label{co:countable-partition-sigma}
   If $r^* \perp \mathcal{L}_{\rho^*}$, $\hat V\geq \overline{V}+\overline \sigma/2$, and there exists a countable partition $\tilde{\observations}$ of $\observations$ such that $r^*$ is constant within each cell of $\tilde{\observations}$ and $(r^*(\tilde{o}): \tilde{o} \in \tilde{\observations})$ is iid then
\begin{align*}
\I(r^*;\hat r) &\ge \frac{1}{\sup_{\tilde{o} \in \tilde{\observations}} p_{\mathrm{att}}(\tilde{o})} \KL(\mathrm{Bern}(\hat V) || \mathrm{Bern}(\overline{V} + \overline \sigma/2)).
\end{align*} 
\end{corollary}
\begin{proof}
   We have by Lemma \ref{lm:bound-expectation-of-positive-component-simplified} that $\overline V_+ \le \overline V + \overline \sigma / 2$. Because $\hat{V} \geq \overline{V} + \overline \sigma/2$, 
   $$
  \KL(\mathrm{Bern}(\hat{V}) || \mathrm{Bern}(\overline{V} + \overline \sigma/2)) \le \KL(\mathrm{Bern}(\hat{V}) || \mathrm{Bern}(\overline{V}_+)).
   $$
   The result follows from Corollary \ref{co:countable-partition-V+}.
\end{proof}

\informationlowerboundgeneralization*
\begin{proof}
We apply \Cref{co:countable-partition-sigma} and observe that $\hat V\geq V^\dagger \geq \overline V + \overline \sigma/2$.  
\end{proof}

\subsection{Proofs for \Cref{sec:limiting-capabilities}}
In this section, we assume $\Pi$ and $\observations$ are finite. For each $\lambda \geq 0$ and $Q \in \Delta_\Pi$, define a regularized version of $\hat{\pi}$:
\begin{align}
\hat{P}_{Q,\lambda} \in \argmax_{P \in \Delta_\Pi} \left(\lambda \sum_{\pi \in \Pi} P(\pi) \sum_{o \in \observations} \rho^*(o|\pi) \hat{r}(o) - \KL(P \| Q) \right).\label{eq:kl-optimization}
\end{align}
\begin{lemma}\label{le:kl-unique-optimizer} For any $Q$ and any $\lambda \geq 0$, the optimization problem \Cref{eq:kl-optimization} has unique solution
$$
\hat{P}_{Q,\lambda}(\pi) = \frac{Q(\pi) e^{\lambda \hat{f}(\pi)}}{Z_\lambda}
$$
with 
$$
Z_\lambda = \sum_{\pi \in \Pi} Q(\pi) e^{\lambda \hat{f}(\pi)},
$$
and 
$$
\hat{f}(\pi) = \sum_{o \in \observations} \rho^*(o|\pi) \hat{r}(o).
$$
\end{lemma}
\begin{proof}
For shorthand, interpret functions of $\pi$ as vectors so that the objective can be written as
$$g_\lambda(P) = \lambda P^\top \hat{f} - \KL(P \| Q).$$

We will now verify that $g_\lambda$ is uniquely maximized by $\hat{P}_{Q,\lambda}.$ Observe that $\hat{P}_{Q,\lambda}(\pi)$ satisfies
\begin{align*}
\log \frac{\hat{P}_{Q,\lambda}(\pi)}{Q(\pi)} =& \lambda \hat{f}(\pi) - \log Z_\lambda,
\end{align*}
and therefore, for any $P$ that is absolutely continuous with respect to $Q$,
\begin{align*}
\log \frac{P(\pi)}{\hat{P}_{Q,\lambda}(\pi)} &= \log \frac{P(\pi)}{Q(\pi)} - \lambda \hat{f}(\pi) + \log Z_\lambda \\
\sum_{\pi \in \Pi} P(\pi) \log \frac{P(\pi)}{\hat{P}_{Q,\lambda}(\pi)} &= \sum_{\pi \in \Pi} P(\pi) \log \frac{P(\pi)}{Q(\pi)} - \lambda \sum_{\pi \in \Pi} P(\pi) \hat{f}(\pi) + \log Z_\lambda \\
\KL(P \| \hat{P}_{Q,\lambda}) &= \KL(P\|Q) - \lambda P^\top \hat{f} + \log Z_\lambda \\
\KL(P \| \hat{P}_{Q,\lambda}) &= -g_\lambda(P) + \log Z_\lambda.
\end{align*}
It follows that
$$g_\lambda(P) = \log Z_\lambda - \KL(P \|\hat{P}_{Q,\lambda}).$$
Since the KL divergence is nonnegative, $g_\lambda(P) \leq \log Z_\lambda$, with equality if and only if $\KL(P \|\hat{P}_{Q,\lambda})=0$.  Since this happens if and only if $P=\hat{P}_{Q,\lambda}$, $\hat{P}_{Q,\lambda}$ is the unique maximizer of $g_\lambda$.
\end{proof}

\begin{lemma}
\label{le:policy-derivative}
For any $Q \in \Delta_\Pi$,
$$\frac{d}{d\lambda} \hat{P}_{Q,\lambda}(\pi) \Big|_{\lambda=0} = Q(\pi) \left(\sum_{o \in \observations} \rho^*(o|\pi) \hat{r}(o) - \sum_{\pi' \in \Pi} Q(\pi') \sum_{o \in \observations} \rho^*(o|\pi') \hat{r}(o)\right).$$
\end{lemma}
\begin{proof}
Differentiating the expression for $\hat{P}_{Q,\lambda}(\pi)$ obtained in \Cref{le:kl-unique-optimizer}, we obtain
\begin{align*}
\frac{d}{d\lambda} \hat{P}_{Q,\lambda}(\pi)
&= \frac{Q(\pi)e^{\lambda \hat{f}(\pi)} \hat{f}(\pi)Z_\lambda - Q(\pi)e^{\lambda \hat{f}(\pi)} \frac{d}{d\lambda} Z_\lambda}{Z_\lambda^2} \\
&= \hat{P}_{Q,\lambda}(\pi) \left(\hat{f}(\pi) - \frac{\frac{d}{d\lambda} Z_\lambda}{Z_\lambda}\right) \\
&= \hat{P}_{Q,\lambda}(\pi) \left(\hat{f}(\pi) - \frac{\sum_{\pi' \in \Pi} Q(\pi') \hat{f}(\pi') e^{\lambda \hat{f}(\pi')}}{Z_\lambda}\right).
\end{align*}
It follows that
\begin{align*}
\frac{d}{d\lambda} \hat{P}_{Q,\lambda}(\pi) \Big|_{\lambda=0}
&= Q(\pi) \left(\hat{f}(\pi) - \sum_{\pi' \in \Pi} Q(\pi') \hat{f}(\pi')\right).
\end{align*}
\end{proof}

The value generated by $\hat{P}_{Q,\lambda}$ is
$$\hat{V}_{Q,\lambda} = \E\left[\sum_{\pi \in \Pi} \hat{P}_{Q,\lambda}(\pi) \sum_{o \in \observations} \rho^*(o|\pi) r^*(o) \right].$$
The value generated by $Q$ is
$$\tilde{V}_Q = \E\left[\sum_{\pi \in \Pi} Q(\pi) \sum_{o \in \observations} \rho^*(o|\pi) r^*(o)\right].$$

\begin{definition}
    If there exists $m \in \mathbb{R}$, $k > 0$ and random vectors $v$ and $w$ such that 
    $$\hat r(o)= kr^*(o) + w_o + v_o r^*(o) + m $$
    and $(v, w) \perp r^*$, $\E[w_o] = \E[v_o] = 0$, and $\E[w_o^2] < \infty$, $\E[v_o^2] < \infty$ for all $o \in \observations$ we say that $\hat r$ was constructed with independent noise and scaling $k$ and shift $m$.
\end{definition}

\begin{lemma}
\label{lm:realized-value-equivalence}
For all $\pi \in \Pi$, let $\hat{f}(\pi) = \sum_{o \in \observations} \rho^*(o|\pi) \hat{r}(o)$ and $f^*(\pi) = \sum_{o \in \observations} \rho^*(o|\pi) r^*(o)$. If $\hat r$ was constructed with independent noise and scaling $k$ and shift $m$, then for all $Q \in \Delta_\Pi$
    $$
        \E\left[(Q^T\hat{f}) (Q^T f^*)\right] = k \E\left[ (Q^Tf^*)^2 \right] + m\E[Q^Tf^*]
    $$
\end{lemma}

\begin{proof}
   We have that
\begin{align*}
\E\left[(Q^T \hat{f}) (Q^T f^*)\right] 
&= \E\left[\left( \sum_{\pi \in \Pi} Q(\pi) \sum_{o \in \observations} \rho^*(o|\pi) \hat{r}(o)\right) \left( \sum_{\pi \in \Pi} Q(\pi) \sum_{o \in \observations} \rho^*(o|\pi) r^*(o)\right)\right] \\
&= \E\left[\left( \sum_{\pi \in \Pi} Q(\pi) \sum_{o \in \observations} \rho^*(o|\pi) (kr^*(o) + w(o) + v(o) r^*(o) + m) \right) \left( \sum_{\pi \in \Pi} Q(\pi) \sum_{o \in \observations} \rho^*(o|\pi) r^*(o)\right)\right] \\
&\stackrel{(a)}{=} \E\left[\left( \sum_{\pi \in \Pi} Q(\pi) \sum_{o \in \observations} \rho^*(o|\pi) (kr^*(o) + m)  \right) \left( \sum_{\pi \in \Pi} Q(\pi) \sum_{o \in \observations} \rho^*(o|\pi) r^*(o)\right)\right] \\
&= \E\left[ k\left( \sum_{\pi \in \Pi} Q(\pi) \sum_{o \in \observations} \rho^*(o|\pi) r^*(o)\right)^2 + m \left( \sum_{\pi \in \Pi} Q(\pi) \sum_{o \in \observations} \rho^*(o|\pi) r^*(o)\right)  \right] \\
\end{align*} 
where step $(a)$ follows due to $\E[w(o)] = 0$ and  $\E[v(o)] = 0$ and $(w,v) \perp (r^*, \rho^*, Q)$.
\end{proof}

\begin{corollary}
\label{cor:realized-value-equivalence}
For all $\pi \in \Pi$, let $\hat{f}(\pi) = \sum_{o \in \observations} \rho^*(o|\pi) \hat{r}(o)$ and $f^*(\pi) = \sum_{o \in \observations} \rho^*(o|\pi) r^*(o)$. If $\hat r$ was constructed with independent noise and scaling $k$ and shift $m$, then for all $\pi \in \Pi$
    $$
        \E\left[\hat{f}(\pi) f^*(\pi)\right] = k \E\left[f^*(\pi)^2 \right] + m\E[f^*(\pi)].
    $$
\end{corollary}
\begin{proof}
For all $\pi \in \Pi$, let $Q_\pi \in \Delta_\Pi$ be a one-hot distribution such that $Q_\pi(\pi) = 1$. Then, the results follows from Lemma \ref{lm:realized-value-equivalence}.
\end{proof}

\begin{lemma}
\label{le:value-derivative}
Suppose $\hat r$ was constructed with independent noise and scaling $k > 0$. Then, for all $Q \in \Delta_\Pi$,
$$\frac{d}{d\lambda} \hat{V}_{Q,\lambda} \Big|_{\lambda=0} = k \E\left[\sum_{\pi \in \Pi} Q(\pi) \left(\sum_{o \in \observations} \rho^*(o|\pi) r^*(o) - \sum_{\pi' \in \Pi} Q(\pi') \sum_{o \in \observations} \rho^*(o|\pi') r^*(o)\right)^2 \right].$$
\end{lemma}
\begin{proof}
For shorthand, let $\hat{f}(\pi) = \sum_{o \in \observations} \rho^*(o|\pi) \hat{r}(o)$ and $f^*(\pi) = \sum_{o \in \observations} \rho^*(o|\pi) r^*(o)$.

By Lemma \ref{le:policy-derivative}, we have that
\begin{align*}
\frac{d}{d\lambda} \hat{V}_{Q,\lambda} \Big|_{\lambda=0} 
&= \E\left[\sum_{\pi \in \Pi} \left(\frac{d}{d\lambda} \hat{P}_{Q,\lambda}(\pi) \Big|_{\lambda=0} \right) f^*(\pi) \right] \\
&= \E\left[\sum_{\pi \in \Pi} Q(\pi) (\hat{f}(\pi) - Q^\top \hat{f}) f^*(\pi)\right] \\
&= \left(\sum_{\pi \in \Pi} Q(\pi) \E[\hat{f}(\pi) f^*(\pi)] \right) - \E[(Q^\top \hat{f}) (Q^\top f^*)] \\
&\stackrel{(a)}{=} \left (\sum_{\pi \in \Pi} Q(\pi) (\E[k f^*(\pi)^2] + m\E[f^*(\pi)]) \right) -  \E[(Q^\top \hat{f}) (Q^\top f^*)] \\
&\stackrel{(b)}{=} \left (\sum_{\pi \in \Pi} Q(\pi) (\E[k f^*(\pi)^2] + m\E[f^*(\pi)]) \right) - \E[k(Q^\top f^*)^2] - m \E[Q^\top f^*] \\
&= k \left( \sum_{\pi \in \Pi} Q(\pi) \E[f^*(\pi)^2] \right) - k\E[(Q^\top f^*)^2] \\
&= k \E\bigg[ \mathrm{Var}_{\tilde \pi \sim Q} ( f^*(\tilde \pi) | \rho^*, r^* ) \bigg].
\end{align*}
Step (a) follows from Corollary \ref{cor:realized-value-equivalence} and step (b) follows from Lemma \ref{lm:realized-value-equivalence}.
\end{proof}

\thmlimitingcapabilities*
\begin{proof}
Let $\hat r = \eta r^* + w$ where $\eta \in (0,1)$ and $w \sim \mathcal{N}(0, I)$ where $w \perp r^*$. Because $|\observations| < \infty$,  $\I(\hat r;r^*)$ can be made arbitrarily small by making $\eta$ small. Note that $\hat r$ was constructed from independent noise.

The condition in Equation \ref{as:uninformed-policy-variation} implies that there exists a pair $(\rho, r)$ with support such that 
$$\min_{\pi \in \Pi_0 }\sum_{o\in\observations}\rho(o|\pi) r(o) < \max_{\pi \in \Pi_0}\sum_{o\in\observations}\rho(o|\pi) r(o).
$$ Thus, $\mathrm{Var}_{\tilde \pi \sim Q} ( f^*(\tilde \pi) | \rho, r ) > 0$ where $\tilde \pi$ is sampled from $Q = P_0$. Since variance is nonnegative, it follows that 
$$\E[\mathrm{Var}_{\tilde \pi \sim Q} ( f^*(\tilde \pi) | \rho^*, r^* ) ] > 0.$$

Since $\hat r$ was constructed with independent noise, it follows from Lemma \ref{le:value-derivative} that $\frac{d}{d\lambda} \hat{V}_{Q,\lambda} \Big|_{\lambda=0} > 0$. The result follows by choosing $Q=P_0$ and noting $\hat V_{\lambda}=\hat V_{P_0,\lambda}$.

\end{proof}

\section{Mutual Information and Bits Required}
\label{sec:mutual-information-bits-required}

In this section, we elaborate on the use of mutual information as a measure of the number of bits about $r^*$ that must be conveyed by the designer.

Here is one way to interpret how $\hat r$ is constructed.
Let $\mathcal{M} = \{0,1\}^*$ denote the set of binary strings, and let $M$ be an $\mathcal{M}$-valued random variable such that $r^* \perp \hat r \mid M$. In other words, we have a Markov chain
\[
r^* \to M \to \hat r .
\]
We interpret $M$ as the message that the designer conveys to the agent, and use $|M|$ to denote its length. To motivate the use of $\I(r^*;\hat r)$ as a measure of the number of bits that must be conveyed about $r^*$, we analyze two communication protocols and relate the message length $|M|$ to $\I(r^*;\hat r)$ in each.

\subsection{Fixed-Length Communication Protocol}

\begin{theorem}
For all $k \in \mathbb{Z}_{++}$ and all $\{0,1\}^*$-valued random variables $M$, if $|M| = k$ with probability $1$ and $r^* \perp \hat r \mid M$, then
\[
k \ge \I(r^*;\hat r).
\]
\end{theorem}

\begin{proof}
We have
\[
k 
\stackrel{(a)}{\ge} \H(M)
\stackrel{(b)}{\ge} \I(M;\hat r)
\stackrel{(c)}{\ge} \I(r^*;\hat r).
\]
Step (a) follows from the standard bound that the entropy of a random variable is at most the logarithm of the cardinality of its support. Step (b) follows from the inequality $\I(M;\hat r) \le \min\{\H(M),\H(\hat r)\}$. Step (c) follows from the data-processing inequality applied to the Markov chain $r^* \to M \to \hat r$.
\end{proof}

\subsection{Variable-Length Communication Protocol}

We now consider the case where the message length may depend on $r^*$. We will bound the expected message length $\E[|M|]$ in terms of mutual information.

We begin by restating a result of \citet{alon2002lower}.
\begin{lemma}
\label{lm:encoding-length-bound}
Let $\mathcal{X}$ be a countable set and let $X$ be an $\mathcal{X}$-valued random variable. For any injective function $f: \mathcal{X} \to \{0,1\}^*$,
\[
\E[|f(X)|] \ge \H(X) - \log(\H(X)+1) - \log e,
\]
where $|f(X)|$ denotes the length of the bitstring $f(X)$.
\end{lemma}

This result states that any one-to-one encoding of a discrete random variable has expected length at least $\H(X) - \log(\H(X)+1) - \log e$. This contrasts with the more familiar result for uniquely decodable (e.g., prefix-free) codes, which states that the expected code length is at least $\H(X)$.

\begin{theorem}
For all $\{0,1\}^*$-valued random variables $M$, if $r^* \perp \hat r \mid M$, then
\[
\E[|M|] \ge \I(r^*;\hat r) - \log\big(\I(r^*;\hat r)+1\big) - \log e.
\]
\end{theorem}

\begin{proof}
Let $f: \{0,1\}^* \to \{0,1\}^*$ denote the identity function. Consider first the case when $\I(r^*;\hat r)\geq 1/\ln2-1$. Then,
\begin{align*}
\E[|M|]
&= \E[|f(M)|] \\
&\stackrel{(a)}{\ge} \H(M) - \log(\H(M)+1) - \log e \\
&\stackrel{(b)}{\ge} \I(r^*;\hat r) - \log\big(\I(r^*;\hat r)+1\big) - \log e.
\end{align*}
Step (a) follows from \Cref{lm:encoding-length-bound}. Step (b) follows because the function $x - \log(x+1)$ is nondecreasing for $x \ge 1/\ln2-1$, and because $\H(M) \ge \I(r^*;\hat r)$. The latter inequality follows from the data-processing inequality applied to the Markov chain $r^* \to M \to \hat r$.

When $0\leq \I(r^*;\hat r)<1/\ln2-1$, the theorem holds trivially because $\E[|M|]$ is nonnegative and $\I(r^*;\hat r)-\log(\I(r^*;\hat r)+1)-\log e<0$. 
\end{proof}

This theorem implies that when $\I(r^*;\hat r)$ is large, it provides, in percentage terms, nearly a lower bound on the expected number of bits that must be communicated by the designer.

\end{document}